%% file: main.tex
\documentclass[pdflatex,sn-mathphys-num]{sn-jnl}

\usepackage{graphicx}
\usepackage{multirow}
\usepackage{amsmath,amssymb,amsfonts}
\usepackage{amsthm}
\newtheorem{proposition}{Proposition}
\usepackage[title]{appendix}
\usepackage{xcolor}
\usepackage{textcomp}
\usepackage{booktabs}
\usepackage{tabularx}
\usepackage{enumitem}
\usepackage{algorithm}
\usepackage{algorithmicx}
\usepackage{algpseudocode}
\usepackage{tikz}
\usepackage{pgfplots}
\pgfplotsset{compat=1.18}
\usetikzlibrary{positioning,arrows.meta,shapes.geometric,fit,calc,backgrounds,decorations.pathreplacing}
\hypersetup{hypertexnames=false}
\usepackage{bookmark}

\graphicspath{{figures/}}

\begin{document}

\input{sections/00_frontmatter}
\input{sections/01_introduction}

\input{sections/02_related_work}
\input{sections/03_method_motionaware}
\input{sections/05_experiments}
\input{sections/06_results_and_analysis}

\input{sections/07_limitations_ethics_repro}
\input{sections/08_conclusion}

\input{sections/09_appendix_protocol}

\backmatter

\bookmarksetup{startatroot}

\bmhead{Acknowledgements}
This work was supported by the National Aeronautics and Space Administration (NASA) through the Oklahoma EPSCoR Research Infrastructure Development (RIG) program (Prime Award No.~80NSSCM0029), under Subaward No.~1-511123-OU2 RIG to the University of Oklahoma, administered via Oklahoma State University.

\section*{Declarations}

\bmhead{Funding}
National Aeronautics and Space Administration (NASA), Oklahoma EPSCoR Research Infrastructure Development (RIG) program, Prime Award No.~80NSSCM0029, Subaward No.~1-511123-OU2 RIG.

\bmhead{Conflict of interest}
The authors declare no competing interests.

\bmhead{Ethics approval and consent to participate}
Not applicable.

\bmhead{Consent for publication}
Not applicable.

\bmhead{Data availability}
The FRED dataset used in this study is publicly available~\citep{fred}.

\bmhead{Materials availability}
Not applicable.

\bmhead{Code availability}
All architecture code, training configurations, evaluation scripts, and the error forensics pipeline are publicly available at \url{https://github.com/INQUIRELAB/SparseVoxelDet}.

\bmhead{Author contribution}
All authors contributed to the conception, design, and analysis of the study. All authors read, revised, and approved the final manuscript.

\bibliography{bib/references_master,bib/references_staging}

\end{document}

%% file: sections/00_frontmatter.tex
\title[SparseVoxelDet]{SparseVoxelDet: Fully Sparse Voxel Networks for Efficient Event-Based Drone Detection}

\author[1]{\fnm{Mohamad Yazan} \sur{Sadoun}}\email{mohamad.yazan@ou.edu}

\author[1]{\fnm{Sarah} \sur{Sharif}}\email{s.sh@ou.edu}

\author*[1]{\fnm{Yaser Mike} \sur{Banad}}\email{bana@ou.edu}

\affil*[1]{\orgdiv{Inquire Lab}, \orgname{University of Oklahoma}, \orgaddress{\city{Norman}, \state{OK}, \country{USA}}}

\abstract{Event cameras excel at detecting small, fast drones, but today's detectors give away their key advantage: they convert the sparse event stream into dense grids and pay dense-processing cost on inputs that are almost entirely empty. We present SparseVoxelDet, to our knowledge the first coordinate-sparse 3D event-voxel bounding-box detector: backbone, feature pyramid, temporal reduction, and detection head all operate directly on coordinate-indexed features, with no dense spatial grid at any stage. Building it exposed a hidden failure mode we name \emph{support inflation}: an input filling a median 0.0652\% of the voxel lattice inflates stage by stage until standard pyramid fusion leaves the detection head locally near-dense. We answer with two ideas. Expansion-free inverse-convolution fusion provably creates no active sites beyond the stored backbone supports, cutting head occupancy from a median 78.88\% to 10.53\%; quality-aligned supervision then recovers more accuracy than preserving sparsity costs. The payoff is measured, not assumed: executing the same trained network densely costs a median 27.5$\times$ the work and 4.65$\times$ the latency across 5{,}000 paired frames, with no frame cheaper dense, while the 6.22M-parameter detector reaches 87.01 AP50 on the FRED drone benchmark, ahead of matched dense controls, and holds its lead on the held-out test partition evaluated once. Sparsity, preserved by construction and supervised well, delivers the efficiency and the accuracy together.}

\keywords{Event camera; Sparse 3D convolution; Drone detection; Object detection; Event-based vision; Sparse voxel representation}

\maketitle

%% file: sections/01_introduction.tex
\section{Introduction}
\label{sec:intro}

Detecting small unmanned aerial vehicles (drones) at low latency is critical for airspace security, yet conventional RGB-based detectors struggle under low-light, high-dynamic-range, and cluttered-sky conditions~\citep{fred}. Event cameras, neuromorphic sensors that asynchronously report per-pixel brightness changes with microsecond latency and $>$120\,dB dynamic range~\citep{lichtsteiner2008128}, are increasingly used for this task~\citep{gallego2020event}. On the FRED drone benchmark, event input outperforms RGB input by more than 50 AP50 points under the published protocol~\citep{fred}: the event stream carries the signal.

Converting that stream into dense frame-like grids, however, discards the sensor's most distinctive property: \emph{extreme sparsity}. On FRED, typical event inputs occupy less than 1\% of the voxel lattice, yet dense detectors process all 409{,}600 pixel positions after resizing to $640{\times}640$. In the LiDAR domain, sparse 3D convolution backbones such as SECOND~\citep{yan2018second} and frameworks including PointPillars~\citep{lang2019pointpillars} and CenterPoint~\citep{yin2021center} progressively reduced this reliance on dense representations, and VoxelNeXt~\citep{voxelnext2023} showed that point-cloud detection can run entirely through sparse 3D convolutions with no dense bird's-eye-view map. Event-camera voxel grids share the same structure, 3D tensors (temporal bins $\times$ height $\times$ width) with low occupancy, which raises our central question: \emph{can an object detector process event data natively in the sparse domain, end-to-end, without constructing a dense spatial feature grid?}

We answer it with SparseVoxelDet, to our knowledge the first documented coordinate-sparse 3D event-voxel bounding-box detector: backbone, feature pyramid, temporal reduction, and detection head all operate on coordinate-indexed features without constructing a dense spatial H$\times$W or T$\times$H$\times$W feature grid. We scope the claim to bounding-box detection deliberately, since sparse 3D event pipelines exist for segmentation and centroid localization~\citep{murray2025propeller}, and graph-based detectors predict boxes from sparse event graphs without a voxel grid~\citep{verma2026egsmv}. A preliminary version of this work appeared as a preprint~\citep{sadoun2026preliminary}; this paper extends it with the source-complete label contract, quality-aligned supervision, and the stage-wise support analysis. Existing event detectors either convert events to dense spatial grids before applying conventional architectures~\citep{perot2020learning,gehrig2023rvt,lv2024rtdetr}, or exploit sparsity at the token or activation level while the underlying tensors remain dense-shaped~\citep{peng2024sast,yang2025smamba,wang2025sparse,fan2024sfod} (Table~\ref{tab:sparsity_taxonomy}). SparseVoxelDet instead reframes event-camera detection as a native sparse-computation problem: a SparseSEResNet backbone extracts multi-scale features at strides 4, 8, and 16 through 3D sparse convolutions with element-wise residual connections and squeeze-and-excitation attention; a sparse feature pyramid fuses the scales into a stride-4 map; temporal max-pooling collapses time; and an anchor-free head~\citep{tian2019fcos} predicts classification, left-top-right-bottom box regression, and box quality at each active position.

Pursuing the question produced a sharper one. Coordinate sparsity is usually justified by input occupancy, but a detector pays its cost in the late stages, not at its input. Measuring the realized support at every stage over the full 103{,}672-frame development-validation split shows occupancy growing monotonically even where the active-site count falls: input occupancy is a median 0.0652\%, the fused pyramid reaches a median 68.86\% under transpose fusion, and the final block realizes a median 24.54 of its 27 kernel taps. The pyramid the detector needs in order to localize is precisely what erodes the sparsity that motivated it; we term this stage-by-stage erosion \emph{support inflation}: the multiscale, detector-level manifestation of the dilation problem known from submanifold sparse convolutions~\citep{graham2018sparse}, here compounded by top-down fusion. The measured anatomy yields the design thesis that organizes the paper: \emph{a coordinate-sparse representation is not a compute guarantee: sparsity must be preserved by construction, and its accuracy must be earned by supervision.}

\input{figures/sparse_vs_dense}

Both halves are addressed. Replacing top-down transpose upsampling with \emph{expansion-free inverse-convolution fusion} creates no sites beyond the stored backbone supports and returns the detection head from a median 78.88\% to a median 10.53\% occupancy. Supervision then does the accuracy work: replacing center-only positive assignment with \emph{quality-aligned supervision} (task-aligned inside-box candidate selection, soft quality targets, and a decoded-IoU quality head) trains the same expansion-free architecture to 87.01 AP50 and 43.14 AP50:95 on our source-complete 147/37 development protocol, a 1.49 AP50 gain over center-only supervision at matched optimizer topology. Architecture, data, labels, evaluator, and optimizer topology are held fixed across the two supervision regimes (Section~\ref{sec:experiments}). The controlled dense comparator (YOLO11n trained on the same source-complete training labels and scored on the same frames with the same evaluator) reaches 77.29 AP50 and 38.12 AP50:95 letterboxed; retrained on the same anisotropic resize our model consumes, it reaches 84.68, leaving a $+$2.33 development margin (Section~\ref{sec:main_results}). Evaluated once on FRED's canonical test partition after every decision was frozen, SparseVoxelDet reads 84.33 AP50 against the control's 79.37 (Section~\ref{sec:main_results}).
Our contributions are threefold:
\begin{enumerate}[leftmargin=*]
  \item \emph{Coordinate-sparse architecture.} We introduce SparseVoxelDet, an event-camera detector whose spatial feature path constructs no dense H$\times$W or T$\times$H$\times$W grid (Table~\ref{tab:sparsity_taxonomy}).
  \item \emph{A measured anatomy of sparsity.} We profile realized support at every stage over the full 37-sequence development-validation split, and exact per-operator work counts over 5{,}000 paired frames spanning 36 of those sequences. Support inflation raises occupancy monotonically from a median 0.0652\% at the input to a median 68.86\% in the fused pyramid, concentrating cost exactly where support has become locally near-dense. Expansion-free inverse-convolution fusion removes this expansion by construction, reducing fusion-stage work by 91.3\% and whole-model work by 70.1\% against \texttt{SparseConvTranspose3d} top-down fusion within our architecture family. A separate, numerically matched dense-execution control then isolates what coordinate-sparse execution itself buys: sparse execution is cheaper on all 5{,}000 paired frames profiled, by a median 27.5$\times$ in work and 4.65$\times$ in latency, and the margin tracks occupancy exactly as the thesis predicts (Section~\ref{sec:efficiency}).
  \item \emph{Quality-aligned sparse supervision.} We show that quality-aligned supervision recovers more accuracy than preserving sparsity costs: task-aligned inside-box assignment with soft quality targets and a decoded-IoU quality head delivers 1.49 AP50 over center-only supervision on identical architecture, data, labels, evaluator, and optimizer topology, making the expansion-free path the most accurate of the sparse configurations ablated (Table~\ref{tab:objective_fusion}). The topology-matched control behind this gain is detailed in Sections~\ref{sec:experiments} and~\ref{sec:limits}.
\end{enumerate}

Two audited artifacts accompany these contributions: the source-complete FRED label-conversion stage, which regenerates every reported label set from the original release, and the stage-wise support profiler behind every occupancy measurement in this paper, which is generic to sparse-convolution networks and can quantify support inflation in any coordinate-sparse pipeline (Section~\ref{sec:repro}).

%% file: figures/sparse_vs_dense.tex

\begin{figure}[t]
\centering
\begin{tikzpicture}[
    cell/.style={minimum size=0.28cm, inner sep=0pt},
]

\begin{scope}
    \node[font=\footnotesize\sffamily\bfseries, anchor=south] at (2.24,4.0) {Dense Processing};

    \foreach \x in {0,...,15} {
        \foreach \y in {0,...,15} {
            \fill[gray!10] (\x*0.28, \y*0.25) rectangle ++(0.26, 0.23);
            \draw[gray!20, line width=0.15pt] (\x*0.28, \y*0.25) rectangle ++(0.26, 0.23);
        }
    }

    \foreach \x/\y in {4/12,5/12,5/11,6/11,6/12,7/11,7/12,8/11,4/13,5/13,
                        9/5,9/6,10/5,13/3,13/4,2/8} {
        \fill[blue!65] (\x*0.28, \y*0.25) rectangle ++(0.26, 0.23);
    }

    \draw[green!60!black, thick, rounded corners=0.5pt]
        (3.5*0.28, 10.5*0.25) rectangle ++(5.5*0.28, 3.5*0.25);
    \node[font=\tiny\sffamily, green!50!black, anchor=south west]
        at (9*0.28, 13.5*0.25) {drone};

    \node[font=\scriptsize\sffamily, text=black!65, anchor=north, align=center]
        at (2.24,-0.3) {409{,}600 pixels processed};
    \node[font=\scriptsize\sffamily, text=black!65, anchor=north, align=center]
        at (2.24,-0.7) {median 0.0652\% of the\\voxel lattice active};
\end{scope}

\node[font=\large\sffamily, black!50] at (5.2,2.0) {$\Longrightarrow$};

\begin{scope}[shift={(6.0,0)}]
    \node[font=\footnotesize\sffamily\bfseries, anchor=south] at (2.24,4.0) {Sparse Processing};

    \draw[gray!15, dashed] (0,0) rectangle (4.48,4.0);

    \foreach \x/\y in {4/12,5/12,5/11,6/11,6/12,7/11,7/12,8/11,4/13,5/13,
                        9/5,9/6,10/5,13/3,13/4,2/8} {
        \fill[blue!65] (\x*0.28+0.13, \y*0.25+0.115) circle (0.1);
    }

    \draw[green!60!black, thick, dashed, rounded corners=0.5pt]
        (3.5*0.28, 10.5*0.25) rectangle ++(5.5*0.28, 3.5*0.25);

    \node[font=\scriptsize\sffamily, draw=blue!30, rounded corners=2pt,
          fill=blue!4, inner sep=3pt, anchor=north] at (2.24, -0.2)
        {$\mathbf{C} \in \mathbb{Z}^{M \times 3}$,\;
         $\mathbf{F} \in \mathbb{R}^{M \times 6}$};

    \node[font=\scriptsize\sffamily, text=blue!65!black, anchor=north, align=center]
        at (2.24,-0.85) {\textbf{Only active voxels ($<$1\% of grid)}};
    \node[font=\scriptsize\sffamily, text=blue!50!black, anchor=north]
        at (2.24,-1.2) {95.88$\times$ fewer input coordinates (median)};
\end{scope}

\end{tikzpicture}
\caption{Dense vs.\ coordinate-sparse processing. Left: a dense detector processes 409{,}600 pixels at $640^2$. Right: SparseVoxelDet input has a median 4{,}272 active coordinates, 95.88$\times$ fewer than dense input pixels on a median-count basis. Under transpose fusion, pyramid expansion and temporal pooling grow support to a median 20{,}194.5 of 25{,}600 stride-4 sites (78.88\%); the expansion-free fusion of Section~\ref{sec:fpn} returns the detection head to a median 10.53\% occupancy, with the full distribution and its tail reported in Section~\ref{sec:efficiency}. These full-development-split, stage-qualified counts describe topology, not FLOPs or memory.}
\label{fig:sparse_vs_dense}
\end{figure}

%% file: sections/02_related_work.tex
\section{Related Work}
\label{sec:related}

\subsection{Event-Based Object Detection}

Event cameras~\citep{lichtsteiner2008128} report pixel-level brightness changes asynchronously, producing streams with microsecond temporal resolution and high dynamic range~\citep{gallego2020event}.

Dense conversion methods. The predominant approach converts events into frame-like representations (event histograms~\citep{detournemire2020large}, time surfaces, or voxel grids~\citep{zhu2019unsupervised}) and applies conventional dense detectors: large-scale histogram-based detection~\citep{perot2020learning}, the recurrent vision transformer RVT~\citep{gehrig2023rvt}, learned event representations~\citep{gehrig2019end}, and the cross-attention RGB+event fusion detector ER-DETR~\citep{erdetr2024}. \citet{messikommer2020event} convert such models into asynchronous sparse ones, but training remains synchronous and detection applies 2D spatial sparse convolutions with a YOLO output layer rather than an end-to-end sparse 3D voxel pipeline. The dense-conversion methods above inherit the cost of dense convolutions, processing every spatial position regardless of event occupancy.

The FRED benchmark~\citep{fred} reports event-input AP50/AP50:95 of 87.68/49.25 for YOLOv11, 85.00/43.40 for Faster R-CNN, 82.05/38.98 for RT-DETR, and 68.42/21.80 for the event-only arm of ER-DETR, whose RGB+event fusion configuration reaches 78.59/32.21. The published record omits the YOLOv11 variant, input resolution, training recipe, and evaluator implementation; we therefore treat 87.68 as an external published reference and base every controlled claim on a dense baseline we train and score ourselves (Section~\ref{sec:experiments}). SMG-UAV~\citep{zhang2026smguav} reruns event-only RVT, SAST, and SMamba under one common but not evaluator-matched protocol, reporting 79.3, 77.9, and 81.7 AP50; its higher 89.3 AP50 result uses RGB+event fusion. A census of the ten scholarly works using FRED\footnote{Search snapshot dated 14 July 2026, enumerating every indexed work citing the FRED dataset paper. Nine of the ten were read in full; no statement or number in this paper rests on the tenth (AHM-Net, ICASSP 2026, \url{https://doi.org/10.1109/ICASSP55912.2026.11464096}).} found, among the works read in full, no independent event-only detector that reproduces or exceeds FRED's 87.68 AP50 under a disclosed protocol.

Sparse event processing. Sparsity enters event-based detectors through three mechanisms, summarized in Table~\ref{tab:sparsity_taxonomy}. \emph{Token sparsity} selects which locations participate in attention while the tensors remain dense-shaped, as in SAST~\citep{peng2024sast} and SMamba~\citep{yang2025smamba}, both retaining dense pyramids and heads. \emph{Activation sparsity} exploits zeros inside dense-format tensors: SCR~\citep{wang2025sparse} reports over 92\% activation sparsity from thresholded activations on standard convolutions, with an explicitly dense recurrent state. \emph{Structural sparsity} stores coordinate-indexed tensors and computes only on occupied sites, as in 2D submanifold convolutions~\citep{messikommer2020event} and the event-graph network AEGNN~\citep{schaefer2022aegnn}, but both terminate in a dense-grid output (verified in the official implementations); the more recent eGSMV~\citep{verma2026egsmv} predicts boxes from a spatiotemporal event multigraph, a structurally sparse graph representation rather than a voxel grid. Closest in representation, \citet{murray2025propeller} run a sparse 3D UNet on $(t,x,y)$ event voxels for propeller segmentation and centroid tracking, producing per-voxel confidences rather than bounding boxes, with no feature pyramid or detection head. To our knowledge, SparseVoxelDet maintains structural sparsity on the 3D event voxel grid through backbone, feature pyramid, temporal reduction, and detection head without constructing a dense spatial H$\times$W or T$\times$H$\times$W feature grid.

\input{tables/tab_sparsity_taxonomy}

Spiking and hybrid methods. Spiking neural networks process events through spike trains, preserving temporal sparsity during computation, including EAS-SNN~\citep{eas_snn}, spiking detection backbones~\citep{cordone2022object}, and SFOD~\citep{fan2024sfod}, and SEW-ResNet~\citep{fang2021sew} contributes the spike-element-wise residual connection that forms the basis of our backbone. These architectures still operate on dense grid tensors: EAS-SNN aggregates events into dense count tensors at input and obtains its best accuracy with non-spiking pyramid and head stages, while SFOD decodes fixed-grid binary spikes to dense floating-point maps before its SSD head. We use the SEW residual pattern with sparse 3D convolutions rather than dense spiking neurons.

\subsection{Sparse 3D Convolutions for Point-Based Detection}

Sparse convolutions~\citep{graham2018sparse}, implemented efficiently in spconv~\citep{spconv2022}, process only occupied voxel positions and are well-established for LiDAR 3D detection~\citep{yan2018second,lang2019pointpillars,yin2021center}, largely displacing direct point-set processing~\citep{qi2017pointnet,qi2017pointnetplusplus}. In 2D, Sparse R-CNN~\citep{sun2021sparse} replaced dense feature maps with learnable proposals but relies on a dense backbone. VoxelNeXt~\citep{voxelnext2023} showed that LiDAR detection can be performed \emph{fully sparsely}, predicting boxes directly from sparse voxel features without dense bird's-eye-view conversion, which motivates testing the same premise on event cameras, where FRED input occupancy is below 1\% (Section~\ref{sec:voxelization}).

Our sparse feature pyramid adapts FPN~\citep{lin2017fpn} to the sparse domain with inverse convolutions that restore coarse features exactly onto stored fine-level supports, propagating features only to positions the backbone already occupies, unlike transpose-convolution upsampling, which dilates the active set at every fusion stage (Section~\ref{sec:fpn}). Support growth under repeated sparse convolutions is the dilation problem that motivated submanifold operators~\citep{graham2018sparse}; what we measure and prevent is its multiscale-decoder form.

\subsection{Label Assignment and Quality Estimation}

Dense-detector research has established that \emph{which} positions are supervised as positives, and \emph{what} score they are trained to predict, matter as much as the architecture: positive-sample selection explains much of the anchor-based/anchor-free gap~\citep{zhang2020atss}, dynamic per-object quotas improve assignment~\citep{ge2021yolox}, and task alignment ties classification confidence to localization quality~\citep{feng2021tood}. On a coordinate-sparse detector the problem changes character: candidates exist only at active sites, small objects may contain few or zero active sites, and the head is pointwise (Section~\ref{sec:head}), so supervision is the only channel through which box quality can be taught. Section~\ref{sec:loss} adapts these ideas to the sparse regime, and Section~\ref{sec:results} shows the resulting gain, to our knowledge the first demonstration that quality-aligned supervision transfers to coordinate-sparse event detection.

\subsection{Drone Detection}

Drones are small (often only tens of pixels), move fast, and operate under variable lighting, where RGB detection struggles with low contrast and motion blur~\citep{fred}. The FRED dataset~\citep{fred} provides synchronized RGB and event data across more than seven hours of annotated recording per modality, spanning varying altitude, weather, and illumination. It publishes a canonical 80/20 recording-level split and a challenging split that keeps the same scenarios on both sides while shifting the data distribution, and establishes detection, tracking, and trajectory-forecasting benchmarks. Under the published protocol, the reported YOLOv11 event-frame result is 87.68 AP50 versus 35.24 AP50 for RGB, and the dense detectors fall into a 34--36 AP50 band on RGB while the best exceeds 87 on events.

%% file: tables/tab_sparsity_taxonomy.tex
\begin{table}[t]
\centering
\caption{Where dense spatial grids enter event-based detectors (and the closest LiDAR analogue). B/N/H: backbone / neck / head operate on \emph{structurally} sparse coordinate-indexed features. ``Input'' means a dense spatial grid is constructed from the raw events before learned computation. $^{\dag}$Dense output verified in the official implementations: sparse features are converted to a dense map before a fully-connected YOLO layer (Messikommer et al.), or graph nodes are max-pooled onto a fixed grid before a linear YOLO output (AEGNN). $^{*}$Sparse concatenation of downsampled stages; no multi-scale pyramid.}\label{tab:sparsity_taxonomy}
\footnotesize
\setlength{\tabcolsep}{4pt}
\begin{tabular}{@{}llcccl@{}}
\toprule
Method & Sparsity mechanism & B & N & H & Dense spatial grid enters \\
\midrule
RVT~\citep{gehrig2023rvt} & -- (dense baseline) & $\times$ & $\times$ & $\times$ & input \\
SAST~\citep{peng2024sast} & token/window masking & $\times$ & $\times$ & $\times$ & input \\
SMamba~\citep{yang2025smamba} & token dropping & $\times$ & $\times$ & $\times$ & input \\
SFOD~\citep{fan2024sfod} & activation (spiking) & $\times$ & $\times$ & $\times$ & input \\
EAS-SNN~\citep{eas_snn} & activation (spiking) & $\times$ & $\times$ & $\times$ & input \\
SCR~\citep{wang2025sparse} & activation (thresholded) & $\times$ & $\times$ & $\times$ & input \\
Messikommer et al.~\citep{messikommer2020event} & structural, 2D submanifold & \checkmark & -- & $\times$ & head$^{\dag}$ \\
AEGNN~\citep{schaefer2022aegnn} & structural, event graph & \checkmark & -- & $\times$ & head$^{\dag}$ \\
VoxelNeXt~\citep{voxelnext2023} (LiDAR) & structural, 3D sparse conv & \checkmark & \checkmark$^{*}$ & \checkmark & none \\
\textbf{SparseVoxelDet (ours)} & structural, 3D sparse conv & \checkmark & \checkmark & \checkmark & none \\
\bottomrule
\end{tabular}
\end{table}

%% file: sections/03_method_motionaware.tex
\section{Method: SparseVoxelDet}
\label{sec:method}

SparseVoxelDet keeps spatial feature computation indexed by active coordinates, constructing no dense spatial H$\times$W or T$\times$H$\times$W feature grid. The \emph{structural support} of a sparse tensor is the set of coordinates stored in its index array, irrespective of whether a stored feature vector is zero; an \emph{active position} is a member of that support. Figure~\ref{fig:architecture} illustrates the pipeline.

\input{figures/architecture_sparsevoxeldet}

\subsection{Event Voxelization}
\label{sec:voxelization}

Given a stream of asynchronous events $\{(x_i, y_i, t_i, p_i)\}$ within a time window, we construct a sparse voxel grid in the discretized-event-volume family of~\citet{zhu2019unsupervised}, hard-binned: each event is assigned to a single temporal bin rather than linearly weighted across two. The temporal axis is discretized into $T$ uniform bins at a spatial resolution of $H \times W$ pixels.

Per-voxel features. Each active voxel at $(t, y, x)$ stores a $C_\text{in}$-channel vector computed from the events in that bin. The minimal representation uses $C_\text{in}=2$ channels, the positive and negative event counts $(c^+, c^-)$; the richer one uses $C_\text{in}=6$ channels that add recency and temporal variation per polarity:
\begin{equation}
    \mathbf{f} = \bigl[\underbrace{\log(1{+}c^+),\;\log(1{+}c^-)}_{\text{count}},\;
    \underbrace{r^+,\; r^-}_{\text{recency}},\;
    \underbrace{\sigma^+_t,\; \sigma^-_t}_{\text{temporal std}}\bigr],
\end{equation}
where $\log(1{+}c^{\pm})$ compresses high-activity counts; $r^{\pm} = \exp(-\lambda (t_{\max} - t_{\text{last}}^{\pm}))$ are exponential recency features, with $t_{\max}$ the maximum event timestamp in the current time window and $\lambda = 5/\Delta t_{\text{range}}$, where $\Delta t_{\text{range}}$ is the timestamp span of the events actually retained in that frame, floored at $1\,\mu$s, so the oldest retained event always decays to $e^{-5} \approx 0.007$ and recency is measured against the frame's own activity. The $\sigma^{\pm}_t$ are temporal standard deviations of event timestamps normalized by the same span. The estimator, the collision and outlier rules, and the conventions for voxels holding one event or a single polarity are read from the pinned voxel-construction source (Section~\ref{sec:repro}).

The resulting representation is a \texttt{SparseConvTensor} with spatial shape $[T, H, W]$ and $C_\text{in}$ channels, containing $M_0$ active voxels. At the native sensor resolution of $1280 \times 720$ with $T=16$ bins the grid holds ${\sim}14.7 \times 10^6$ possible positions, of which $<$1\% are typically occupied. Section~\ref{sec:preprocessing} details the input configuration used.

\subsection{SparseSEResNet Backbone}
\label{sec:backbone}

The backbone is a sparse 3D convolutional residual network whose element-wise residual connections follow SEW-ResNet~\citep{fang2021sew}, adapted to operate entirely through sparse convolutions~\citep{he2016deep}; it uses no spiking neurons, the residual pattern being adopted for its gradient-friendly properties in deep sparse networks. The \texttt{nano\_deep} configuration has a base channel width of 32 with block depths $[2, 2, 2, 1]$ across four stages, and squeeze-and-excitation (SE) attention~\citep{hu2018squeeze} on the final two stages, whose squeeze averages over each sample's own active sites rather than over the batch.

Architecture. A sparse convolutional stem with stride $(1, 2, 2)$ maps the $C_\text{in}$-channel input to 32 channels. Four stages of sparse residual blocks follow with per-stage strides $(1, 2, 2, 2)$, producing cumulative spatial strides 2, 4, 8, and 16 from the original input, with channels 32, 64, 128, 256. Each residual block consists of two sparse convolution layers, both \texttt{SubMConv3d} (submanifold sparse convolution) in stride-1 blocks, while in a block that changes stride the second is a strided \texttt{SparseConv3d}, with per-site layer normalization over the channel vector and ReLU activation, connected by an element-wise residual addition on the sparse feature tensors. Normalizing each active site independently avoids the batch and spatial statistics that a sparse tensor's varying support makes ill-defined; the head uses a different normalization (Section~\ref{sec:head}). At stride boundaries, where input and output dimensions differ, a $1 \times 1$ sparse convolution with stride acts as the shortcut.

The backbone outputs multi-scale features at three levels: $\{C_2, C_3, C_4\}$ with stride $\{4, 8, 16\}$ and channels $\{64, 128, 256\}$. Spatial features remain coordinate-indexed as sparse strided convolutions transform the active set.

\subsection{Expansion-Free Sparse Feature Pyramid}
\label{sec:fpn}

We adapt the Feature Pyramid Network~\citep{lin2017fpn} to the sparse domain, fusing multi-scale features into a single stride-4 map for detection of small objects. How coarse features reach fine resolution controls whether the pyramid preserves the support sparsity the backbone delivers.

Lateral connections. Each backbone output $C_k$ is projected to a common channel dimension ($d=128$) through a $1 \times 1$ sparse submanifold convolution followed by the backbone's per-site normalization, also used in the upsampling and output blocks:
\begin{equation}
    P_k = \text{SparseNorm}(\text{SubMConv3d}_{1\times1}(C_k)), \quad k \in \{2, 3, 4\}.
\end{equation}

Expansion-free top-down pathway. Features flow from coarse to fine through sparse \emph{inverse} convolutions (\texttt{SparseInverseConv3d}), each paired to one of the backbone's stride-$(1, 2, 2)$ downsampling convolutions through a shared indice key that fixes its geometry. An inverse convolution restores features exactly onto the stored support of the corresponding earlier stage, creating no output site that the backbone did not already occupy. Fused features at each level are obtained via sparse element-wise addition:
\begin{align}
    P_3' &= P_3 + \text{SparseInverse}(P_4), \\
    P_2' &= P_2 + \text{SparseInverse}(P_3').
\end{align}
The fused active set is therefore bounded by the stored backbone supports, a property of coordinate sets rather than of learned weights.

\begin{proposition}[Support preservation]
\label{prop:support}
Fix a batch element and let $S_k \subset \mathbb{Z}^3$ be the structural support of backbone level $k$. Let $D_k$ denote the cached indice data of the strided convolution from level $k{-}1$ to level $k$: its ordered input and output coordinate rows and their pair tables. If the input to \texttt{SparseInverseConv3d} carries its coordinate rows in the cached forward-output order of $D_k$, the layer reuses the reversed pair tables and sets its output coordinate rows to the cached forward-input rows; hence $\mathrm{coords}(\mathrm{SparseInverse}_{D_k}(X)) = S_{k-1}$, independent of its weights. Define the top-down recursion $Q_4 = P_4$ and $Q_{k-1} = P_{k-1} + \mathrm{SparseInverse}_{D_k}(Q_k)$ for $k = 4, 3$, where both addition operands are aligned to the cached $S_{k-1}$ row order and the fused tensor is rebuilt on that order carrying the $D_{k-1}$ indice data. Since the lateral $1{\times}1$ submanifold projection preserves coordinate rows, $\mathrm{coords}(P_{k-1}) = S_{k-1}$ and each fused level has structural support exactly $S_{k-1}$; the fused stride-4 support equals $S_2$.
\end{proposition}

The proof, with two scope remarks (on which operator's indice data defines the coordinate map, and on the proposition formalizing the library's indice-cache contract rather than a library-independent fact), appears in Appendix~\ref{app:proof}.

For any fusion operator at level $k$ with nonempty fused support we define the \emph{support-inflation factor} $\iota_k = |\mathrm{coords}(\text{fused})| / |S_k|$ and the \emph{extraneous-support fraction} $\varepsilon_k = |\mathrm{coords}(\text{fused}) \setminus S_k| / |\mathrm{coords}(\text{fused})|$. Proposition~\ref{prop:support} gives $\iota_k = 1$ and $\varepsilon_k = 0$ for the fused levels $k \in \{2, 3\}$ by construction; transpose fusion can yield $\iota_k > 1$ because every overlapping kernel tap generates an output site, and Section~\ref{sec:efficiency} measures its realized cumulative effect on the detection head. As an implementation certificate we assert coordinate equality between the fused pyramid and the stored backbone supports at runtime over the development-validation split (103{,}672 frames checked, zero violations observed).

Transpose-fusion baseline. The conventional sparse alternative upsamples with transpose convolutions (\texttt{SparseConvTranspose3d}), which generate an output site for every kernel tap that overlaps an occupied input site, so support \emph{expands} at each fusion stage. Section~\ref{sec:efficiency} measures the consequence: under transpose fusion the post-pool detection head reaches a median 78.88\% occupancy, against a median 10.53\% under expansion-free fusion, a 7.5$\times$ gap. We retain the transpose variant as a controlled baseline in Section~\ref{sec:results}.

Output refinement. A $3 \times 3$ sparse submanifold convolution refines the fused stride-4 features into the final map $F \in \mathbb{R}^{M \times d}$, where $M$ is the number of active positions. We fuse to a single stride-4 level rather than maintaining multi-scale outputs because small targets dominate FRED: three quarters of matched targets measure at most 41.07~px in maximum dimension (Section~\ref{sec:qualitative}), and at stride~4 the grid remains fine enough to resolve them ($320 \times 180$ at native $1280 \times 720$, $160 \times 160$ at the $640^2$ primary resolution).

\subsection{Temporal Squeeze}
\label{sec:temporal_squeeze}

The backbone and FPN produce 3D sparse features with indices $(b, t, y, x)$, while detection requires 2D predictions. We collapse the temporal dimension by grouping all active voxels that share the same $(b, y, x)$ position and max-pooling over their features:
\begin{equation}
    f_{(b,y,x)} = \max_{t \in \mathcal{T}_{(b,y,x)}} F_{(b,t,y,x)},
\end{equation}
where $\mathcal{T}_{(b,y,x)}$ is the set of temporal bins with active features at spatial position $(y, x)$. This produces $M'$ unique 2D positions with $d$-dimensional features and constructs no dense H$\times$W grid. Max-pooling preserves the strongest activation across time, appropriate for transient moving objects whose event signature may be concentrated in a few temporal bins.

\subsection{Detection Head}
\label{sec:head}

The detection head follows the anchor-free FCOS design~\citep{tian2019fcos}, predicting classification, box regression, and box quality at each active 2D position. Since the temporally-pooled features are a flat tensor of $M'$ position vectors rather than a dense grid, the head is a shared MLP (equivalent to pointwise $1 \times 1$ convolutions) applied per position; spatial context has already been aggregated by the backbone and the FPN.

Shared trunk. Two layers of Linear $\rightarrow$ GroupNorm (8 groups) $\rightarrow$ ReLU, both with hidden dimension $d=128$. This is the one place in the network where features are grouped for normalization (compare Section~\ref{sec:backbone}).

Prediction branches. Three linear heads applied in parallel:
\begin{itemize}[leftmargin=*]
    \item Classification: a raw logit $\hat{c} \in \mathbb{R}^{M' \times 1}$ whose sigmoid is the class probability, initialized so that $\sigma(\hat{c}) = p_0 = 0.01$.
    \item Box regression: $\hat{b} \in \mathbb{R}^{M' \times 4}$, predicting LTRB distances from the position center. The exponential of the raw predictions gives pixel distances.
    \item Box quality: a raw logit $\hat{z} \in \mathbb{R}^{M' \times 1}$, read through a sigmoid. Under quality-aligned supervision (Section~\ref{sec:loss}) this branch is trained to predict the decoded box's IoU with its assigned ground truth; under the center-only baseline it predicts FCOS centerness.
\end{itemize}

Box decoding. Each active position $(y, x)$ at stride $s=4$ has region center $(x \cdot s + s/2,\; y \cdot s + s/2)$, and the LTRB predictions define the box as:
\begin{equation}
    \text{box} = \left[c_x - d_l,\; c_y - d_t,\; c_x + d_r,\; c_y + d_b\right],
\end{equation}
where $(c_x, c_y)$ is the position center and $(d_l, d_t, d_r, d_b) = \exp(\min(\hat{b}, 10))$ are the predicted left, top, right, and bottom distances (LTRB) from the center to the box edges; the clamp before exponentiation prevents half-precision overflow, and decoded boxes are clamped to the image extent.

\subsection{Training Objective}
\label{sec:loss}

The training loss combines three terms with fixed weights $\lambda_\text{cls}{=}1$, $\lambda_\text{reg}{=}2$, $\lambda_\text{ctr}{=}1$:
\begin{equation}
    \mathcal{L} = \lambda_\text{cls} \mathcal{L}_\text{cls} + \lambda_\text{reg} \mathcal{L}_\text{reg} + \lambda_\text{ctr} \mathcal{L}_\text{ctr}.
\end{equation}
Regression is a convex combination of GIoU loss~\citep{rezatofighi2019giou} and a normalized Wasserstein distance term~\citep{wang2021nwd}, $(1{-}w)\,\mathcal{L}_\text{GIoU} + w\,\mathcal{L}_\text{NWD}$ with $w{=}0.5$ and constant $c{=}12.8$; the latter remains informative for the small boxes that dominate FRED where IoU-based losses saturate. What differs between our two supervision regimes is \emph{which positions are trained as positives and what targets they receive}.

Center-only baseline. The conventional FCOS-style assignment marks a position positive when it falls near a ground-truth box center; classification uses focal loss~\citep{lin2017focal} ($\alpha{=}0.25$, $\gamma{=}2.0$) with binary targets and the quality branch regresses positional centerness. This trains the score to reflect \emph{where a position sits}, not \emph{how good its decoded box is}, an anchor-free analogue of the classification--localization misalignment seen in dense detectors~\citep{feng2021tood}.

Quality-aligned supervision. Our final objective aligns supervision with decoded-box quality in three coupled changes, holding architecture, data, labels, evaluator, learning rate, warm-up length, nominal 20-epoch budget, and per-epoch sample exposure fixed, with optimizer topology matched by the rerun of Sections~\ref{sec:experiments} and~\ref{sec:limits}:
\begin{itemize}[leftmargin=*]
    \item \emph{Task-aligned inside-box assignment.} Candidate positives are the active sites inside each ground-truth box, ranked by a detached alignment score $a = \hat{s}^{\,\alpha} \cdot \text{IoU}(\hat{b}, g)^{\beta}$ combining classification confidence $\hat{s}$ and the decoded box's IoU with the ground truth $g$ ($\alpha{=}1$, $\beta{=}6$). A per-object dynamic quota, the floor of the summed IoU of the object's ten highest-IoU candidates (all of them if fewer), clamped to at least one and at most the candidate count, selects the positives. A position claimed by multiple objects is resolved by deferred acceptance, so the outcome does not depend on the order objects appear in the annotation file. What the rule cannot supply is a positive for an object whose interior holds no active site at all: 0.59\% of the boxes the objective received at the selected epoch, against zero boxes losing every candidate to a conflict and a quota fill ratio of exactly one. Both per-epoch counters are part of the released training log (Section~\ref{sec:repro}).
    \item \emph{Soft quality targets.} Selected positives receive soft classification targets proportional to their normalized alignment rather than binary labels, trained through a binary quality focal loss~\citep{li2020gfl} ($\alpha{=}0.25$, modulation exponent 2.0) whose modulating factor is the distance between prediction and soft target. The target equation, its normalization scope, and its detachment rule are read from the pinned objective source (Section~\ref{sec:repro}). During the first two epochs, while decoded boxes are still uninformative, the target \emph{value} bootstraps rather than the assignment rule: the positive class target is blended from a hard label toward its quality value (hard at epoch~0, an equal blend at epoch~1, quality alone from epoch~2) while inside-box task-aligned assignment is in force from the first step.
    \item \emph{Decoded-IoU quality head.} The quality branch is retrained from positional centerness to the detached IoU of each positive's decoded box, so the inference-time product $\sigma(\hat{c}) \cdot \sigma(\hat{z})$ estimates object confidence times box quality, the quantity NMS and AP ranking need.
\end{itemize}
No inference-time component changes. Section~\ref{sec:results} compares this contribution against the center-only baseline on identical architecture, data, labels, and evaluator, epoch by epoch as well as at the selected checkpoint.

\subsection{Inference and NMS}
\label{sec:inference}

At inference, the detection score is $\text{score} = \sigma(\hat{c}) \cdot \sigma(\hat{z})$, predictions below $\tau = 0.05$ are filtered, and non-maximum suppression is applied per frame with IoU threshold 0.5, capped at 100 detections. Algorithm~\ref{alg:inference} summarizes the forward pass.

\begin{algorithm}[!ht]
\caption{SparseVoxelDet Inference}
\label{alg:inference}
\scriptsize
\begin{algorithmic}[1]
\Require Event stream $\mathcal{E} = \{(x_i, y_i, t_i, p_i)\}$ within time window $\Delta t$
\Ensure Set of detections $\mathcal{D} = \{(\text{box}_j, \text{score}_j)\}$
\State \textbf{Voxelization:}
\State $V \gets \textsc{SparseVoxelGrid}(\mathcal{E}, T, H, W)$ \Comment{$[M_0, C_\text{in}]$ sparse tensor}
\State \textbf{Backbone (SparseSEResNet):}
\State $V_0 \gets \textsc{SparseStem}(V)$ \Comment{stride $(1,2,2)$, 32ch}
\For{$k = 1, 2, 3, 4$}
    \State $V_k \gets \textsc{SparseResStage}_k(V_{k-1})$ \Comment{strides 1,2,2,2; channels 32,64,128,256}
    \If{$k \geq 3$}
        \State $V_k \gets \textsc{SqueezeExcite}(V_k)$ \Comment{SE in stages 3--4}
    \EndIf
\EndFor
\State \textbf{Expansion-free SparseFPN (top-down fusion):}
\For{$k = 2, 3, 4$}
    \State $P_k \gets \textsc{SparseNorm}(\textsc{SubMConv}_{1\times1}(V_k))$ \Comment{Project to 128ch; SparseNorm is channelwise LayerNorm over active sites, not batch normalization}
\EndFor
\State $P_3' \gets P_3 + \textsc{SparseInverse}(P_4)$ \Comment{restores stored stride-8 support}
\State $P_2' \gets P_2 + \textsc{SparseInverse}(P_3')$ \Comment{restores stored stride-4 support}
\State $F \gets \textsc{SubMConv}_{3\times3}(P_2')$ \Comment{Refined stride-4 features, $[M, 128]$}
\State \textbf{Temporal squeeze:}
\State $F' \gets \textsc{MaxPool}_t(F)$ \Comment{$[M', 128]$, collapse time dim}
\State \textbf{Detection head:}
\State $H \gets \textsc{MLP}(F')$ \Comment{2-layer shared trunk}
\State $\hat{c}, \hat{b}, \hat{z} \gets \textsc{Cls}(H),\; \textsc{Box}(H),\; \textsc{Qual}(H)$
\State $\text{score} \gets \sigma(\hat{c}) \cdot \sigma(\hat{z})$
\State $\text{boxes} \gets \textsc{DecodeLTRB}(\hat{b}, \text{positions}, \text{stride}=4)$
\State \textbf{Post-processing:}
\State $\mathcal{D} \gets \textsc{NMS}(\text{boxes}[\text{score} > 0.05],\; \text{IoU}=0.5,\; \text{max}=100)$
\State \Return $\mathcal{D}$
\end{algorithmic}
\end{algorithm}

\subsection{Design Discussion}
\label{sec:design_discussion}

The coordinate-sparse premise adapts VoxelNeXt~\citep{voxelnext2023} to event data, but three differences shape the design. The third axis is time, not depth, so the time axis must be collapsed (Section~\ref{sec:temporal_squeeze}) rather than projected to a bird's-eye view. Occupancy is motion-dependent rather than surface-dependent, and event voxels carry brightness-change statistics rather than point geometry, motivating the 6-channel encoding (Section~\ref{sec:voxelization}). Where VoxelNeXt serves large road objects with a single sparse concatenation of downsampled stages, small drones concentrate box evidence at object contours, so we adopt a top-down sparse pyramid (Section~\ref{sec:fpn}) that restores stride-4 resolution before prediction without manufacturing support.

Model size. The full model contains 6.22M parameters.

%% file: figures/architecture_sparsevoxeldet.tex

\begin{figure*}[t]
\centering
\includegraphics[width=\textwidth]{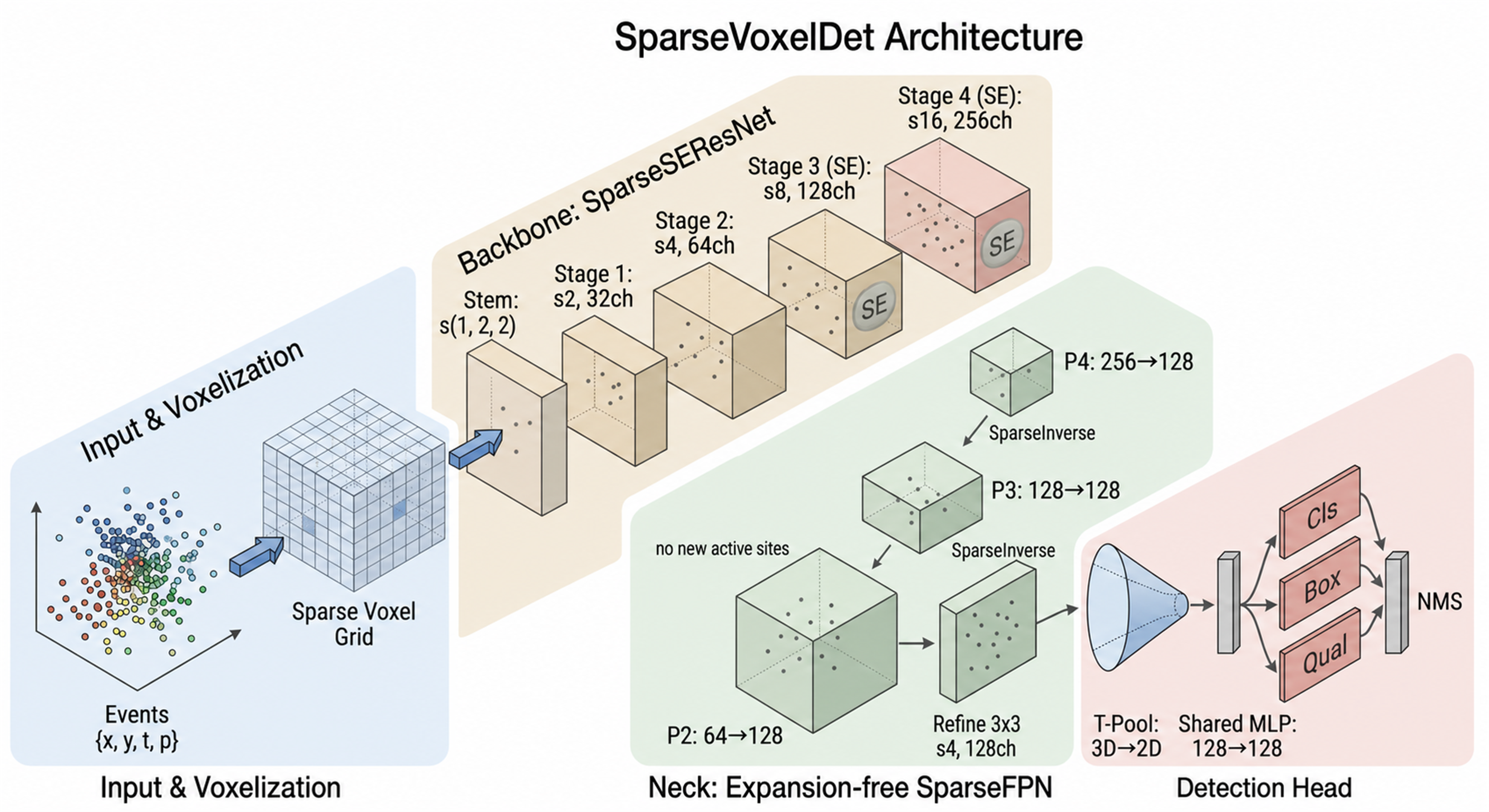}
\caption{SparseVoxelDet architecture.
Events are binned into a sparse 3D voxel grid ($T{\times}H{\times}W$, 6 temporal-surface channels).
The SparseSEResNet backbone processes voxels through four stages of sparse 3D residual blocks
with SE attention at stages~3--4, producing multi-scale features \{$C_2$, $C_3$, $C_4$\}
at strides 4, 8, 16. The expansion-free sparse pyramid fuses scales top-down via sparse lateral projections
and inverse convolutions that restore coarse features onto the stored backbone supports,
yielding a unified stride-4 feature map without creating new active sites.
After temporal max-pooling collapses the time axis, the detection head
(shared MLP) predicts classification, LTRB box regression, and box quality
at each active 2D position, followed by NMS.
The spatial feature path constructs no dense $H{\times}W$ or $T{\times}H{\times}W$ feature grid.}
\label{fig:architecture}
\end{figure*}

%% file: sections/05_experiments.tex
\section{Experiments}
\label{sec:experiments}

\subsection{Dataset: FRED Benchmark}
\label{sec:dataset}

We evaluate on the FRED (Florence RGB-Event Drone) dataset~\citep{fred}, the largest event-camera benchmark for drone detection, providing synchronized RGB and event-camera recordings. The dataset is captured with a Prophesee EVK4 (IMX636) event camera at a native resolution of 1280$\times$720, co-registered with an RGB sensor, and provides more than seven hours of annotated recording per modality. Five commercial and custom drone models serve as targets: Betafpv Air75, DarwinFPV CineApe20, DJI Tello EDU, DJI Mini~2, and DJI Mini~3.

Sequences cover diverse operational conditions: varying altitude (${\sim}$5--80\,m), illumination (daylight, dusk, night), weather (clear, overcast, rain), and background complexity (open sky, buildings, vegetation).

Event cameras maintain usable signal across the full range of lighting conditions, including complete darkness where RGB cameras produce unusable frames. Figure~\ref{fig:fred_overview} illustrates this with paired RGB and event frames spanning daylight, dusk, and near-darkness. The panels are qualitative: accuracy is reported pooled over the development-validation split and is not broken down by illumination condition.

\begin{figure}[t]
\centering
\includegraphics[width=\linewidth]{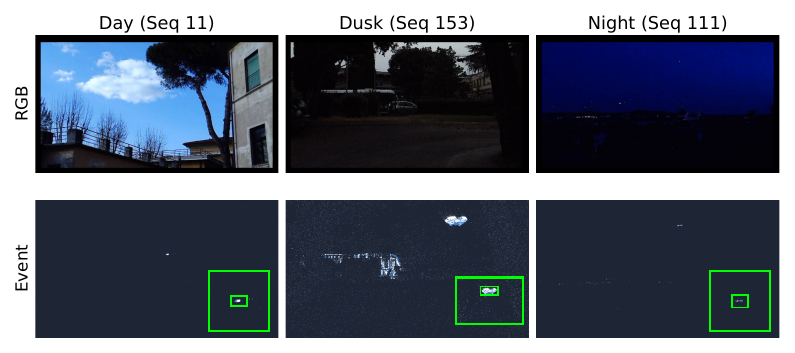}
\caption{FRED dataset samples across lighting conditions. Top row: RGB frames. Bottom row: the dataset's own event-frame renderings, in which events appear as light blue and white marks against a dark background; the corner inset enlarges the target region, with the ground-truth box drawn inside it. While RGB quality degrades from day to night, the event stream retains a visible drone signature in each. Qualitative illustration; no per-condition accuracy is claimed.}
\label{fig:fred_overview}
\end{figure}

FRED's canonical split partitions the benchmark into a 184-sequence training partition and a disjoint held-out test partition, and publishes no validation partition. We therefore adopt FRED's own split and carve within it: every development result in this paper is obtained on an \emph{internal development protocol} derived from those 184 canonical training sequences alone, which carry 510{,}373 RGB--event paired frames by our own count: we carve them into a 147-sequence training split (406{,}701 paired frames) and a disjoint 37-sequence development-validation split (103{,}672 paired frames), recording-level and non-overlapping. The benchmark's canonical test split was held out through every selection, ablation, and diagnostic decision in this work, and evaluated exactly once afterwards (Section~\ref{sec:main_results}). Development numbers remain non-comparable to results computed on that split. Event frames are provided at 30\,FPS, temporally aligned with RGB frames.

FRED's distributed annotations and reference loader preserve all rows that share a timestamp, while our earlier conversion retained one row per timestamp (a defect in our conversion, not in FRED), so we rebuilt every label directly from the distributed source. The rebuild restored 23{,}774 boxes in the 147-sequence training split, from 409{,}650 to 433{,}424, and 18{,}310 second boxes in the 37-sequence development-validation split, from 103{,}672 to 121{,}982 boxes. We call the result the \emph{source-complete} label contract (complete up to a known residue of frames whose second box could not be disambiguated: 1{,}935 development frames, 1.87\%, and 1{,}448 test frames, 1.21\%, affecting both splits in near-equal measure): every checkpoint in this paper is trained on the source-complete training-label root and scored against the source-complete development-validation labels, with the exact label manifests and hashes recorded alongside the released artifacts.

\subsection{Event Preprocessing}
\label{sec:preprocessing}

The preprocessing pipeline converts raw asynchronous event streams into sparse voxel grids for 3D sparse convolution. It supports two spatial resolutions sharing identical architecture and 6-channel temporal surface features, and every controlled result in this paper is measured at the first:

\begin{itemize}[leftmargin=*]
    \item $640 \times 640$ (primary): $C_\text{in}{=}6$ temporal surface features (Section~\ref{sec:voxelization}), $T{=}16$ temporal bins, spatially resized to $640 \times 640$, activity-outlier voxel filtering (within each 33\,ms window, voxels whose event count exceeds three times that window's own mean events-per-voxel, with the threshold floored at three events, are suppressed).
    \item Native $1280 \times 720$: Same 6-channel features and $T{=}16$ bins, at the sensor's native resolution (no spatial rescaling). The pipeline supports this configuration and we release its voxelized data, but no result reported in this paper is measured at it.
\end{itemize}

Common pipeline. The continuous event stream is divided into non-overlapping windows of 33\,ms (one frame at 30\,FPS), aligned with FRED annotation timestamps. Windows are half-open, so an event whose timestamp falls exactly on a boundary belongs to the later frame; the window origin and the boundary search that resolves it follow the pinned voxel-construction source (Section~\ref{sec:repro}). Within each window, timestamps are made frame-relative and quantized into $T$ uniform temporal bins:
\begin{equation}
    t_\text{bin} = \mathrm{clip}\!\left(\left\lfloor \frac{t \cdot T}{\Delta t} \right\rfloor,\; 0,\; T{-}1\right), \quad t_\text{bin} \in \{0, 1, \ldots, T{-}1\},
\end{equation}
where $t$ is the event's offset within its own window.
Only voxels with at least one event are stored. The output is a pair: integer coordinates $\mathbf{C} \in \mathbb{Z}^{M \times 3}$ (ordered $[t, y, x]$) and features $\mathbf{F} \in \mathbb{R}^{M \times C_\text{in}}$, where $M$ is the number of active voxels, wrapped into a \texttt{SparseConvTensor}~\citep{spconv2022}. The preprocessing path constructs no dense spatial H$\times$W or T$\times$H$\times$W feature grid; spatial data flow as coordinate--feature pairs (Figure~\ref{fig:preprocessing}).

Occupancy. Table~\ref{tab:sparsity} reports stage-wise active-site counts and occupancy, measured over all 103{,}672 frames of the 37-sequence development-validation split with the profiler described in Section~\ref{sec:repro}; Section~\ref{sec:efficiency} analyzes the full distributions.

\input{tables/sparsity_statistics}

\input{figures/preprocessing_pipeline}

\subsection{FRED Baselines}

FRED's published detection baselines, YOLOv11~\citep{jocher2024ultralytics}, RT-DETR~\citep{lv2024rtdetr}, Faster R-CNN, and ER-DETR~\citep{erdetr2024}, all process event frames as dense 3-channel images; none uses the native sparse structure of the event data. Their published values (Section~\ref{sec:related}) are computed on FRED's canonical test partition under undisclosed recipes and evaluator implementations, so they serve as context for our sealed test rows only. Table~\ref{tab:main_results} therefore carries two blocks, development and sealed test, each containing only rows we control end to end under one set of frames, labels, evaluator, and post-processing settings.

FRED reports 87.68 AP50 and 49.25 AP50:95 for event-frame YOLOv11 on its canonical benchmark partition, without the variant, resolution, recipe, evaluator implementation, code, or weights. We therefore retain 87.68 as an external published reference and contribute the controlled dense reference that FRED currently lacks: a YOLO11n control trained on the same 147-sequence split and source-complete training-label contract, scored on the same 37-sequence frames and source-complete development-validation labels through the frozen evaluator used for SparseVoxelDet. We release that row with its full recipe (Section~\ref{sec:repro}), label contract, evaluator, and checkpoint hashes. The control reaches 77.29 AP50 and 38.12 AP50:95 and appears in Table~\ref{tab:main_results}, which carries controlled rows only.

Both arms are event-only: the dense control consumes rendered event frames, while SparseVoxelDet consumes coordinate--feature pairs and never materializes a frame. The comparison therefore contrasts dense-frame with coordinate-sparse \emph{processing} of the same event stream under shared frame indices, labels, evaluator, and post-processing, rather than a difference in sensing modality. The two input pipelines pair a label index with adjacent 33\,ms windows under slightly different timing conventions (Section~\ref{sec:repro}), so the underlying sensor evidence is index-aligned rather than identical; the comparison does not isolate processing as a single factor, and Section~\ref{sec:limits} states each remaining asymmetry.

\subsection{Training Details}

SparseVoxelDet is trained under a 20-epoch schedule (epochs numbered 0--19) with a 5{,}000-optimizer-step linear warmup followed by a fully decaying cosine schedule and EMA weights. For each run, the immutable selection rule chooses the highest development AP50, breaks a tie by AP50:95, and then chooses the earlier epoch. The rule selects epoch 5 of the quality-aligned arm, released as \texttt{epoch\_005.pt} (SHA-256 \texttt{b25c62a0}, 99{,}757{,}547 bytes), evaluated and released with its EMA shadow weights. Its in-training selection trace climbs from 81.73 AP50 at epoch~0 to 86.97 at epoch~5 (the offline re-score carried by the tables is 87.01, Section~\ref{sec:error_forensics_results}); Figure~\ref{fig:training_curves} shows the trace for both supervision arms. That arm's final epoch was stopped by the training controller's power watchdog and produced no validation row, so every epoch-matched comparison in this paper runs over the nineteen scored epochs.

\begin{itemize}[leftmargin=*]
    \item Optimizer: AdamW~\citep{loshchilov2019decoupled} with learning rate $3 \times 10^{-4}$ and weight decay $1 \times 10^{-2}$.
    \item Batch size: 2 per GPU. The center-only baseline arm trains on one GPU with gradient accumulation over 4 steps (effective batch 8); the quality-aligned arm trains with 3-GPU distributed data parallelism (global batch 6, no accumulation).
    \item Precision: Mixed-precision training (FP16) with dynamic loss scaling.
    \item EMA: Exponential moving average of weights with decay 0.9997.
    \item Gradient clipping: Max norm 5.0 with NaN gradient sanitization.
    \item Loss weights: $\lambda_\text{cls} = 1.0$, $\lambda_\text{reg} = 2.0$, $\lambda_\text{ctr} = 1.0$; focal loss $\alpha = 0.25$, $\gamma = 2.0$. Both follow FCOS~\citep{tian2019fcos}.
    \item Augmentation: Horizontal flip ($p{=}0.5$), polarity inversion ($p{=}0.3$, permuting the 6-channel feature pairs), random event dropout ($p{=}0.02$), and a random affine applied to every sample (scale $\times$0.8--1.2 about the image centre, translation up to 10\% of each axis). Each training batch independently draws its input size from \{576$^2$, 608$^2$, 640$^2$\}; validation always runs at $640^2$. The full augmentation order and the two rules under which an annotated box is dropped before the loss sees it are read from the pinned dataset source (Section~\ref{sec:repro}).
    \item Seeds: The development protocol uses seed 42 for architecture and objective selection, plus one replication seed trained under the frozen final configuration after all design decisions were locked. The two runs select at adjacent epochs and land within $0.63$ AP50 of one another: seed 42 at epoch 5 with $87.01$ AP50 and $43.14$ AP50:95, seed 123 at epoch 6 with $86.38$ and $42.26$. We give both values rather than a mean and a standard deviation, which two runs do not support.
\end{itemize}

Hardware identifiers, durations, epoch counts, optimizer-update counts, and checkpoint-selection outcomes for both arms are in the frozen run records. The seed-42 quality-aligned run scored 19 epochs of its 20-epoch schedule and selected at epoch 5, after 406{,}498 optimizer updates. The seed-123 run is a truncated partial replication: it scored 16 of its 20 scheduled epochs and stopped entering epoch 16 on a deterministic mixed-precision gradient-scaler overflow. It selected at epoch 6, and its values are labeled as a partial replication wherever they appear. The two supervision arms match on per-epoch forward-pass exposure to within two samples and on update-contributing samples to within 0.051\% of an epoch; the per-epoch counters behind both figures are in the released run records. The original arms' optimizer topologies differed (one accumulating GPU against three distributed ranks), so the headline supervision comparison uses a center-only arm retrained under the quality-aligned arm's exact distributed configuration; its run record, abort receipt, and offline scoring are frozen alongside the others, and Section~\ref{sec:limits} discloses its truncation. Figure~\ref{fig:training_curves} compares the original arms epoch by epoch.

The dense control's own recipe is released in full with the code (Section~\ref{sec:repro}). Two entries are resolved by the framework rather than requested by us, and both are stated in that release. The requested \texttt{auto} optimizer resolves, for a schedule of this length, to SGD at learning rate $0.01$ and momentum $0.9$, overriding the requested momentum and zeroing the requested warmup bias rate. The framework also releases the checkpoint with the highest AP50:95 alone, where SparseVoxelDet's rule takes the highest AP50 and breaks ties on AP50:95. Applying our rule to the control's own validation trace would release epoch 12 rather than epoch 22, and epoch 12 scores 0.09 AP50 higher on the framework's internal evaluator, so the asymmetry runs slightly against the control and cannot account for the margin.

\subsection{Evaluation Protocol}

We report COCO-style mAP@50 and mAP@50:95 computed over all 103{,}672 frames of the 37-sequence development-validation split. At inference, the score threshold is 0.05 (the frozen, study-specific score floor; absolute AP values are specific to this operating point), NMS IoU threshold is 0.5, and the maximum number of detections per frame is 100. Validation runs every epoch over the full development-validation split using the EMA weights. The sealed-test rows of Table~\ref{tab:main_results} apply the same evaluator, the same operating point, and the same EMA weights, once, to the 119{,}459-frame canonical test split under source-complete label fingerprint \texttt{84ab36d5}. Margins are computed at full precision before rounding and may differ by 0.01 from the difference of the two displayed values.

\subsection{Error Forensics Pipeline}
\label{sec:error_forensics}

We run an error-forensics pass on the selected checkpoint's cached development-validation predictions: per-frame predictions with scores, matched ground-truth boxes, IoU values, and TP/FP/FN labels, stratified by IoU, confidence rank, object size, event density, and sequence. The pass is a replay rather than a re-run: it reproduces the frozen evaluator's decisions bit-exactly at every one of the ten IoU thresholds, and its matched ground-truth set size equals its reported true-positive count. Its frozen output, summarized in Section~\ref{sec:error_forensics_results}, is released with the code.

%% file: tables/sparsity_statistics.tex
\begin{table}[t]
\centering
\footnotesize
\caption{Stage-qualified SparseVoxelDet topology at $640^2$, measured over all 103{,}672 frames of the 37-sequence development-validation split. Occupancy is the realized active-site share of each stage's own T$\times$H$\times$W site count; these counts are not interpreted as FLOPs or memory. Read down the occupancy column: the median rises at every stage from the input to the fused pyramid, which is support inflation measured rather than asserted. The final row is the expansion-free head, which replaces the row above it and is not a further step in that ladder. The input carries a second, coarser basis used in the text (its median 4{,}272 active coordinates are ${\sim}$95.88$\times$ fewer than the 409{,}600 pixels of the $640^2$ lattice, at a mean of 14{,}582.66), and dividing those two numbers returns 1.04\%, which is a coordinate count against a pixel lattice and not an occupancy.}\label{tab:sparsity}
\begin{tabular}{@{}lrrrrr@{}}
\toprule
 & Active sites & Stage sites & \multicolumn{3}{c}{Occupancy (\%)} \\
\cmidrule(l){4-6}
Stage & (median) & (reference) & median & p95 & max \\
\midrule
Input volume & 4{,}272 & 6{,}553{,}600 & 0.0652 & 0.77 & 5.73 \\
Stem & 14{,}266 & 1{,}638{,}400 & 0.87 & 11.91 & 33.82 \\
Backbone C2 (stride 4) & 19{,}271 & 409{,}600 & 4.70 & 53.33 & 80.72 \\
Backbone C3 (stride 8) & 16{,}867 & 102{,}400 & 16.47 & 92.92 & 99.86 \\
Backbone C4 (stride 16) & 11{,}742 & 25{,}600 & 45.87 & 99.94 & 100 \\
Fused pyramid (pre-pool), transpose & 282{,}039 & 409{,}600 & 68.86 & 99.59 & 99.96 \\
Post-pool support, transpose & 20{,}194.5 & 25{,}600 & 78.88 & 99.89 & 100 \\
\midrule
Post-pool support, expansion-free & 2{,}695 & 25{,}600 & 10.53 & 80.67 & 97.67 \\
\bottomrule
\end{tabular}
\end{table}

%% file: figures/preprocessing_pipeline.tex

\begin{figure*}[t]
\centering
\resizebox{\textwidth}{!}{%
\begin{tikzpicture}[
    box/.style={draw, rounded corners=3pt, minimum height=1.15cm, minimum width=2.0cm, font=\small\sffamily, align=center, thick, inner sep=4pt},
    note/.style={font=\small\sffamily, align=center},
    arr/.style={-{Stealth[length=2.5mm, width=1.8mm]}, thick, black!65},
    branch/.style={-{Stealth[length=2.5mm, width=1.8mm]}, thick, dashed, black!50},
]

\node[box, fill=blue!10] (raw) {Raw events\\$\{x_i,y_i,t_i,p_i\}$};
\node[box, fill=blue!10, below=1.3cm of raw] (slice) {33\,ms windows\\30\,FPS alignment};
\draw[arr] (raw) -- (slice);

\node[box, fill=green!10, right=0.9cm of raw] (resize) {Sparse resize\\$1280{\times}720 \rightarrow 640{\times}640$\\anisotropic};
\node[box, fill=green!10, right=0.55cm of resize] (voxel) {$T{=}16$ temporal bins\\activity-outlier filter\\$c>\max(3,3\bar c)$};
\node[box, fill=green!10, right=0.55cm of voxel] (features) {Six temporal-surface\\features per\\active voxel};
\node[box, fill=red!8, draw=red!40, right=0.55cm of features] (sparse) {Coordinate--feature pairs\\$\mathbf C\in\mathbb Z^{M\times3}$, $\mathbf F\in\mathbb R^{M\times6}$};
\draw[arr] (slice.east) -- ++(0.35,0) |- (resize.west);
\draw[arr] (resize) -- (voxel);
\draw[arr] (voxel) -- (features);
\draw[arr] (features) -- (sparse);
\node[note, below=3pt of sparse] {SparseVoxelDet; no dense spatial feature grid};

\node[box, fill=orange!12, right=0.9cm of slice] (render) {Dense event image\\native $1280{\times}720$ render};
\node[box, fill=orange!12, right=0.55cm of render] (letterbox) {YOLO letterbox\\$640{\times}360$ signal\\$280$ padding rows};
\node[box, fill=orange!12, right=0.55cm of letterbox] (dense) {Dense $640{\times}640$\\event tensor};
\draw[arr] (slice) -- (render);
\draw[arr] (render) -- (letterbox);
\draw[arr] (letterbox) -- (dense);
\node[note, below=3pt of dense] {YOLO11n dense event-only control};

\node[box, fill=gray!8, below=1.05cm of slice] (native) {Native-resolution sparse\\configuration $T{\times}720{\times}1280$\\no spatial resize};
\draw[branch] (slice) -- (native);
\node[note, below=3pt of native] {Released; unmeasured here};

\end{tikzpicture}%
}
\caption{Event preprocessing and comparison paths. Both arms use non-overlapping 33\,ms event windows of the same nominal duration, paired to labels under slightly different indexing conventions (Appendix~\ref{app:repro}). The primary SparseVoxelDet path anisotropically maps native $1280{\times}720$ coordinates onto a $640{\times}640$ lattice, quantizes timestamps into $T{=}16$ bins, suppresses voxels whose count exceeds $\max(3,3\bar c)$ for that window, computes six temporal-surface features, and retains only coordinate--feature pairs. The pipeline also supports a native-resolution sparse configuration, whose voxelized data we release, but no result reported in this paper is measured at it. The dense event-only control renders the same event stream at native aspect ratio and letterboxes it to $640{\times}640$, yielding $640{\times}360$ signal plus 280 padding rows. Section~\ref{sec:limits} accounts for the input-sampling asymmetry between the two arms.}
\label{fig:preprocessing}
\end{figure*}

%% file: sections/06_results_and_analysis.tex
\section{Results and Analysis}
\label{sec:results}

\subsection{Main Results}
\label{sec:main_results}

Table~\ref{tab:main_results} reports the controlled comparison on FRED: an identified checkpoint, source-complete development labels, a frozen frame set, one evaluator, and one post-processing configuration. The control is a system comparator rather than an architecture-matched baseline: SparseVoxelDet carries 2.4$\times$ its parameters (6.22M against 2.6M) and receives 640 signal rows to the control's 360 under its isotropic letterbox, while the control is COCO-pretrained and selects from a longer schedule, so Section~\ref{sec:limits} states each asymmetry's direction. Holding frames, source-complete labels, and evaluator fixed, the dense control reaches 77.29 AP50 and 38.12 AP50:95, and SparseVoxelDet reaches 87.01 AP50 and 43.14 AP50:95. Both arms are scored on the same 103{,}672 development-validation frames, so the margins are paired by construction: $+$9.72 AP50 and $+$5.02 AP50:95. Two seeds do not support an interval estimate, so we report dispersion distribution-free over the 37 paired sequences instead. Re-scoring the same two frozen prediction caches sequence by sequence, the per-sequence paired margin spans $-$12.78 to $+$34.50 AP50 with median $+$3.83, and $-$30.48 to $+$17.33 AP50:95 with median $+$3.06; SparseVoxelDet leads on 34 of the 37 sequences at AP50 (two-sided sign test, $p = 1.2\times10^{-7}$) and on 32 of 37 at AP50:95 ($p = 7.4\times10^{-6}$). Average precision is not additive over sequences, so the pooled margin is not the mean of these per-sequence values.

The resolution-matched control measures how much of that margin the input-sampling asymmetry explains, and the answer is most of it. Retraining the same YOLO11n on the identical anisotropic $640^2$ resize our model consumes (removing the 640-versus-360 signal-row asymmetry while keeping COCO pretraining, the longer schedule, and the capacity difference) raises the dense control from 77.29 to 84.68 AP50 and from 38.12 to 42.44 AP50:95 on the same frozen development split, under the checkpoint-selection rule preregistered before scoring (per-epoch evaluation of all 100 epochs through our evaluator, maximum AP50, epoch 24 selected). The paired development margin against this stronger control is $+$2.33 AP50 and $+$0.70 AP50:95, so the sparse system still leads when the dense arm consumes the same input sampling, though architecture, pretraining, and capacity still differ. This control is a single seed and a development-only measurement; the sealed test was spent before it was scored, so it has no test row.

A controlled audit of one frozen YOLO11n prediction cache (permissive setting: confidence 0.001, NMS 0.7, 300 detections) separates the published discrepancy into protocol effects. Each factor is changed on its own against the same 82.42 permissive baseline: replacing the source-complete labels with the one-box-per-timestamp legacy contract lowers its AP50 from 82.42 to 79.96, and, separately, re-scoring the 82.42 baseline under the frozen SparseVoxelDet post-processing settings (threshold 0.05, class-agnostic NMS 0.5, 100 detections) lowers it to 77.06, so labels contribute 2.46 points and the operating point 5.36; the two one-factor effects were not composed. The published 87.68 therefore remains an uncontrolled external reference.

The evaluator was arbitrated rather than assumed. Scoring one frozen dense cache three ways gives 77.29 AP50 under this paper's evaluator, 76.64 under a third-party COCO implementation, and 84.29 under the Ultralytics routine, whose seven-point outlier is an integration artifact (an interpolated ramp to zero precision at recall one). We report the in-house evaluator because it is the one frozen into training and selection; it sits 0.66 AP50 \emph{above} the third-party implementation on the dense cache and 0.45 above on the sparse cache, so both arms are displaced together and the evaluator choice moves the controlled margin by at most the 0.21-point difference between the two offsets, far below the margin itself.

The sealed test, evaluated once. After every decision governing the released models, evaluator, and operating point was frozen, the two released SparseVoxelDet checkpoints and the dense control were each evaluated a single time on FRED's canonical test partition (119{,}459 frames under source-complete label fingerprint \texttt{84ab36d5}, the same evaluator, and the same operating point) under a preregistration written before the split was touched. SparseVoxelDet reaches 84.33 AP50 and 40.19 AP50:95 on seed 42 and 83.01 and 39.16 on the truncated replication seed (Section~\ref{sec:experiments}), against 79.37 and 37.00 for the dense control: margins of $+$4.96 and $+$3.64 AP50 over the control, on data that neither model nor any selection decision ever saw. The two seeds are two values; no across-seed statistic is estimable from them.

Crossing the partition boundary reports the development-to-test change, a compound of selection bias, partition difficulty, and distribution shift that a single crossing cannot decompose, and the two arms move in opposite directions. The sparse detector loses 2.68 AP50 from development to test (87.01 to 84.33) while the dense control gains 2.08 (77.29 to 79.37), so the pooled margin narrows from 9.72 to 4.96 AP50 and from 5.02 to 3.19 AP50:95. The direction of the control's shift bears on the concern that our carved development split favors our own method: the split we carved is the harder of the two for the baseline, and the ordering survives on both seeds. The sealed row also calibrates the external axis: scoring the dense control's one frozen test prediction set through Ultralytics' own metric function instead of ours returns 85.98 AP50 against our evaluator's 79.37 (a 6.6-point displacement on identical predictions), evidence that the gap between our controlled dense row and FRED's published 87.68 is largely a protocol effect rather than a model effect.

\input{tables/tab_main_results}

\subsection{Supervision and Fusion Ablation}
\label{sec:ablation_objective}

Table~\ref{tab:objective_fusion} compares the paper's two design responses on identical data, the same nominal schedule, and one evaluator. The supervision comparison is topology matched: a center-only arm retrained under the quality-aligned arm's exact distributed configuration (3-GPU data parallelism, global batch 6, the same learning rate, warm-up, and cosine schedule) reaches its best development AP50 of 85.52 at epoch 6, scored offline under the uniform protocol of Section~\ref{sec:experiments}, so the supervision gain is 1.49 AP50 with optimizer topology held fixed. The earlier center-only arm trained on one accumulating GPU reaches 85.71, within 0.19 of the topology-matched arm, so the configuration difference the original comparison carried is small and favored the baseline. The matched rerun stopped at epoch 16 of its 20-epoch schedule on a deterministic mixed-precision overflow, so 85.52 is the best of its 16 scored epochs (Section~\ref{sec:limits}). The table varies one design change at a time rather than crossing them: the factorial cell crossing quality-aligned supervision with transpose fusion was not trained. Holding supervision fixed at the center-only objective, replacing transpose fusion with expansion-free inverse-convolution fusion cuts the support the head must process from 78.88\% to 10.53\% median occupancy by construction, and costs 0.60 AP50 (86.31 to 85.71); an expansion-free operator restores coarse features only onto supports the backbone already stored, and under a positional center-only objective that restriction constrains what the head can learn. Moving from center-only to quality-aligned supervision on the expansion-free architecture delivers 1.49 AP50 at matched topology, recovering the fusion cost and adding to it: the final configuration finishes 0.70 AP50 and 2.60 AP50:95 above the transpose baseline while processing roughly a seventh of its head support. The expansion-free path is thus the most accurate of the four trained sparse arms.

\input{tables/tab_objective_fusion}

Figure~\ref{fig:training_curves} shows the development-validation trajectories of both supervision arms under the frozen selection rule (highest AP50, ties broken by AP50:95, then the earlier epoch). Score--quality alignment for the selected checkpoint is visible directly in the score distributions: matched detections carry a mean confidence of 0.483 (median 0.502, 95th percentile 0.656) against 0.122 for unmatched ones (median 0.078, 95th percentile 0.388), and matched detections average 0.758 IoU with their assigned ground truth. The two distributions overlap in the low-confidence band, where the operating point of Section~\ref{sec:error_forensics_results} does its work.

\input{figures/training_curves}

\subsection{Error Forensics}
\label{sec:error_forensics_results}

The forensics pass of Section~\ref{sec:error_forensics} re-scores the selected checkpoint from a cached prediction set rather than re-reading the training log, and it reproduces that checkpoint's AP50 to within 0.05 points: 106{,}767 matched ground-truth boxes against the 106{,}766 recorded during training, one borderline detection out of 121{,}982 crossing the 0.50 matching boundary between runs. Tables~\ref{tab:main_results} and~\ref{tab:objective_fusion} carry the cached offline re-score, from which every derived margin in this paper is computed; the in-training selection trace appears in Section~\ref{sec:experiments}, and the 0.05-point difference between the two is a measured run-to-run reproducibility band for sparse-convolution inference.

Precision, recall, and F1 are reported at the confidence maximizing F1 over the evaluator's fixed grid, $\tau^\star = 0.25$, whereas average precision integrates the full detection set above the 0.05 score floor. At $\tau^\star$ the pass reports 106{,}767 true positives, 18{,}286 false positives, and 15{,}215 false negatives, for recall 0.875, precision 0.854, and F1 0.864 under the frozen evaluator.

Every false negative is attributed to exactly one cause under a fixed priority order, so the components sum to the total. Of the 15{,}215 false negatives, 10{,}285 (67.6\%) are objects the detector found but scored below the reporting threshold, 3{,}304 (21.7\%) are localized above threshold but below the 0.50 matching IoU, 1{,}319 (8.7\%) were removed by non-maximum suppression, 302 (2.0\%) are complete misses with no overlapping prediction at any cached score, and 5 are ranking conflicts in which a sufficient box existed but was out-scored (Figure~\ref{fig:error_breakdown}). The dominant term is therefore scoring rather than detection failure. Complete failures to respond to an object at all account for 2\% of misses and 0.25\% of all ground-truth boxes (Figure~\ref{fig:iou_fn_breakdown}). Only 118 of 103{,}672 frames yield no surviving detection at all, and no frame reaches the 100-detection cap.

\input{figures/error_forensics}

Stratified analyses by confidence band, object size, scene density, and sequence are released with the forensics artifact. Per-sequence recall spans 0.519 to 0.994, so the aggregate figures average over sequences of substantially different difficulty. Average precision falls from 0.870 at the 0.50 matching threshold to 0.820 at 0.55, 0.748 at 0.60, 0.652 at 0.65, 0.530 at 0.70, and 0.383 at 0.75. That profile identifies localization precision, not detection, as the binding constraint over this sweep: the residual error is boundary accuracy on targets a few pixels across.

\input{figures/iou_threshold_analysis}

\subsection{The Anatomy of Sparsity}
\label{sec:efficiency}

Coordinate sparsity is usually justified by input occupancy. Our stage-wise measurements over all 103{,}672 development-validation frames show why that justification is incomplete, and where the design must intervene.

Support inflation. At $640^2$ the input contains a median 4{,}272 active coordinates (mean 14{,}582.66), a median 0.0652\% of the voxel lattice and roughly 95.88$\times$ fewer coordinates than the 409{,}600-pixel dense input. (The percentage is against the full $16{\times}640{\times}640$ lattice, the count ratio against the $640^2$ pixel grid a dense 2D detector would process; in raw counts, 4{,}272 input coordinates become 282{,}039 fused-pyramid sites under transpose fusion against 2{,}695 head sites under expansion-free fusion.) Every strided sparse convolution dilates the active set, and top-down fusion compounds it: under transpose fusion the fused three-dimensional pyramid reaches a median 68.86\% occupancy, temporal pooling leaves a median 20{,}194.5 of 25{,}600 stride-4 head sites (78.88\% occupancy), and the final block realizes a median 24.54 of its 27 kernel taps. The support is locally near-dense precisely where the widest channels execute: low input occupancy does not bound compute, because the detector pays where support has inflated, not where the sensor was sparse.

Because a deployed detector is bounded by its busy frames, we report the full occupancy distribution rather than the median alone: under transpose fusion the post-pool head support is 78.88\% at the median (mean 75.90\%), 99.89\% at the 95th percentile, and 100\% at the maximum; under expansion-free fusion it is 10.53\% at the median (mean 28.70\%), 80.67\% at the 95th percentile, and 97.67\% at the maximum. Expansion-free fusion is therefore sparser at every reported statistic, but the reduction factor falls from 7.5$\times$ at the median to 1.24$\times$ at the 95th percentile and 1.02$\times$ at the maximum. The head-occupancy distribution is right-tailed, and the two fusion modes converge only on the densest frames, where no coordinate-sparse method can retain sparsity the scene does not offer.

Expansion-free fusion intervenes by construction. Inverse-convolution fusion (Section~\ref{sec:fpn}) restores coarse features only onto stored backbone supports, returning head occupancy from 78.88\% to 10.53\% at the median. Within our architecture family (identical frames, hardware, and precision, the two arms differing only in their two top-down fusion operators), this cuts realized fusion-stage multiply--accumulates by 91.3\% (95\% CI 84.1--95.2\%) and whole-model realized multiply--accumulates by 70.1\% (95\% CI 55.7--80.1\%). The whole-model figure is the weaker of the two because the arms share a bit-identical backbone that already accounts for a median 78.5\% of the expansion-free model's multiply--accumulates and cannot improve. Whole-model latency falls by 47.6\% (95\% CI 36.7--54.8\%). Work and latency are profiled over 5{,}000 paired frames from 36 of the 37 development-validation sequences, three timed repeats each; intervals are sequence-cluster percentile bootstraps over paired per-frame ratios (medians of ratios, so dividing released aggregate counts does not reproduce them). Incremental peak allocated memory falls by 83.7\% on a narrower 30-frame panel at batch size 1, reported as a paired 5th--95th percentile (44.6--92.9\%) and treated as the least precise of the four cost statistics. These are within-family measurements against the transpose variant; efficiency scope against the dense control is stated in Section~\ref{sec:limits}, and Appendix~\ref{app:protocol} defines the reporting protocol behind these support-efficiency statistics.

What coordinate sparsity buys. The comparison above answers which fusion operator is cheaper, not what coordinate-sparse \emph{execution} buys over dense execution, because both arms are sparse. We transplanted the selected checkpoint's entire operator set (15 submanifold, 7 strided sparse, and 2 inverse sparse convolutions) into dense \texttt{conv3d} equivalents carrying the identical weights, and gated the transplant before profiling: on ten frames in FP32 with TensorFloat-32 disabled in both the matrix-multiply and convolution paths, the worst head deviation between the two arms is $8.6\times10^{-5}$ against a $10^{-3}$ tolerance, and the realized coordinate sets are identical. Both arms run at batch size 1 under CUDA autocast FP16 on one RTX 3090, over the same 5{,}000 paired frames, three timed repeats each; the gate was not re-established at the shared FP16 profile precision, so the attribution rests on the FP32 gate together with that shared precision, and precision- or input-dependent divergence at FP16 is bounded by that evidence rather than excluded. The dense arm's latency includes the support generation and masking that materializing a dense grid from a sparse input requires, so the reported ratio is between complete forward passes rather than between convolution kernels alone.

Sparse execution costs less on every axis measured. Taking the median of the 5{,}000 paired per-frame ratios, dense execution of the same network performs 27.5$\times$ the logical multiply--accumulates (95\% sequence-cluster CI 9.50--37.62$\times$), takes 4.65$\times$ the latency (95\% CI 2.63--5.27$\times$), and holds 11.4$\times$ the peak allocated memory, the memory figure on the same narrower 30-frame panel used above, reported without an interval. No frame in the sample is faster or cheaper dense: the smallest per-frame latency ratio anywhere in the run is 1.14$\times$ and the smallest work ratio is 2.12$\times$, and restricting to the 250 busiest frames (50{,}432 to 351{,}421 active input coordinates) leaves the minimum unchanged and the median at 1.28$\times$. Nor does the result depend on charging the dense arm for its setup: subtracting its support generation and masking (a median 11.1\% of its forward time) moves the latency ratio to 4.13$\times$, the smallest per-frame ratio to 1.08$\times$, and no frame of the 5{,}000 turns faster dense.

The size of the benefit is a function of occupancy. Binning the same frames by active input coordinates into deciles, the median latency ratio falls monotonically from 7.05$\times$ on the sparsest decile to 1.41$\times$ on the densest, and the median work ratio from 131.5$\times$ to 2.98$\times$. The consequence for a frame budget is direct: at 30\,FPS, 3{,}084 of the 5{,}000 frames (61.7\%) complete inside 33.3\,ms sparse, against 0 of 5{,}000 dense. We state that fraction rather than a frame rate, because the sparse arm's 90th-percentile latency is 106.8\,ms. The measurement is bounded to one card, one checkpoint, 36 of 37 sequences, and 191 to 351{,}421 active input coordinates.

Two estimators are defensible here and they do not agree. Every dense-versus-sparse execution ratio reported above is a median of paired per-frame ratios, which describes the typical frame. The alternative, dividing the pooled total for one arm by the pooled total for the other, describes the whole workload as a batch, and on this run it returns 8.58$\times$ for work and 2.47$\times$ for latency. A pooled total is dominated by the busiest frames, which are exactly where sparse execution has least advantage, so the pooled estimator answers a question about aggregate throughput while the paired median answers one about a frame. We name the estimator wherever a ratio appears.

External reference points, on the one axis that transfers. The most recent efficiency table published on FRED is SMG-UAV's embedded benchmark~\citep{zhang2026smguav}, which reports parameters, theoretical GFLOPs, frame rate, and latency for eleven detectors on an NVIDIA Jetson Orin Nano at batch size 1. Exactly one of those four quantities compares to ours without qualification: parameter count, a static architectural property independent of hardware, input resolution, evaluation split, and profiler. SparseVoxelDet carries 6.22\,M parameters against a published range of 18.5 to 86.5\,M, so it is smaller than every system in that table and 2.97$\times$ smaller than the smallest: RVT, itself an event-only design. That statement carries no accuracy claim and no runtime claim.

The other three axes we decline. Compute: their column is a theoretical cost at a fixed input resolution, while ours is measured and varies 171$\times$ over the 5{,}000 frames (3.29 to 563.69\,GMAC, median 43.40). Latency: their timing scope matches our profiler's, but the hardware differs (an RTX 3090 against an embedded board). Accuracy: the single sealed evaluation of Section~\ref{sec:main_results} puts SparseVoxelDet at 84.33 AP50 on the same canonical test partition where SMG-UAV's protocol reports event-only RVT at 79.3, SAST at 77.9, and SMamba at 81.7; those anchors sit below our sealed row as context, not as an evaluator-matched comparison, since our own dense control moves 6.6 AP50 between two metric implementations on identical predictions. SMG-UAV's higher 89.3 is an RGB+event fusion result and does not enter an event-only comparison.

\subsection{Qualitative Results}
\label{sec:qualitative}

Figures~\ref{fig:qual_tp} and~\ref{fig:qual_failures} show the selected checkpoint on the development-validation split at the operating point $\tau^\star = 0.25$. Each cell pairs the event frame with the RGB frame at the same annotation index and overlays ground truth and predictions. Cells were chosen by a rule fixed before any image was rendered: Figure~\ref{fig:qual_tp} takes one matched detection from each quartile of ground-truth target size, the instance whose IoU is closest to its quartile median; Figure~\ref{fig:qual_failures} takes one instance from each of the four principal false-negative causes of Section~\ref{sec:error_forensics_results}, the instance whose target size is closest to its cause median. We stratify true positives by target size rather than by illumination because FRED carries no lighting annotation; the per-sequence attribute the dataset does provide is the drone model, which each cell reports. Targets span 11.7 to 180.7~px in maximum dimension, with quartile edges 11.73, 26.69, 32.85, 41.07, and 180.69~px.

The failure panel makes the error decomposition concrete. Its four causes account for 15{,}210 of the 15{,}215 false negatives (99.97\%): 10{,}285 boxes that a detection localized correctly but scored below $\tau^\star$ (67.6\%), 3{,}304 localized below the 0.50 matching threshold (21.7\%), 1{,}319 suppressed by non-maximum suppression (8.7\%), and 302 with no overlapping candidate at all (2.0\%). The last cohort has a measured correlate. Counting active pixels inside each ground-truth box across all 16 time bins, all 302 complete misses contain zero, whereas an equal-count control drawn deterministically from the matched detections (equal in cohort size, not matched by target size) contains a minimum of 14 and a median of 232. The two quantities are computed independently (the cause from detector output, the count from the input). We read this as an evidence limit rather than a detection failure. The reading is bounded: extending the count to the two annotation indices on either side, 166 of the 302 boxes remain empty across the whole window, while the remaining 136 hold events at a neighboring index and are equally consistent with a target that briefly stopped generating contrast.

\begin{figure*}[p]
\centering
\includegraphics[scale=0.84]{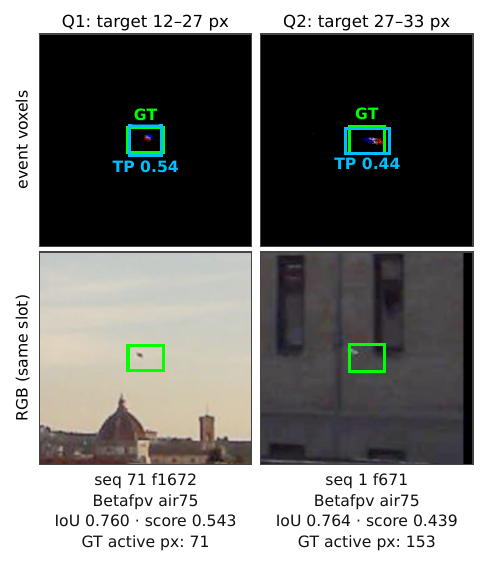}\\[2pt]
\includegraphics[scale=0.84]{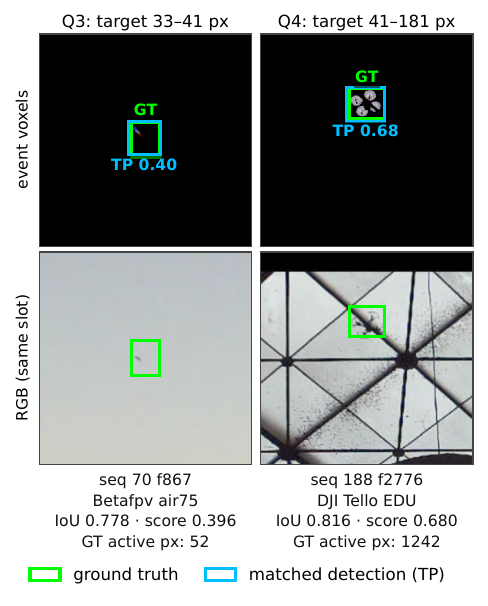}
\caption{Matched detections from the selected checkpoint on the development-validation split, one per quartile of ground-truth target size, arranged in two rows of two. Each cell pairs the accumulated event frame (top) with the RGB frame at the same annotation index (bottom); green: ground truth, cyan: prediction. Annotations give the sequence, frame, drone model, matched IoU, detection score, and the number of active pixels inside the ground-truth box. Cells were selected by the size-quartile rule of Section~\ref{sec:qualitative}, not chosen for appearance; they illustrate the operating point and are not a sample designed to represent the error distribution.}
\label{fig:qual_tp}
\end{figure*}

\begin{figure*}[p]
\centering
\includegraphics[width=\textwidth]{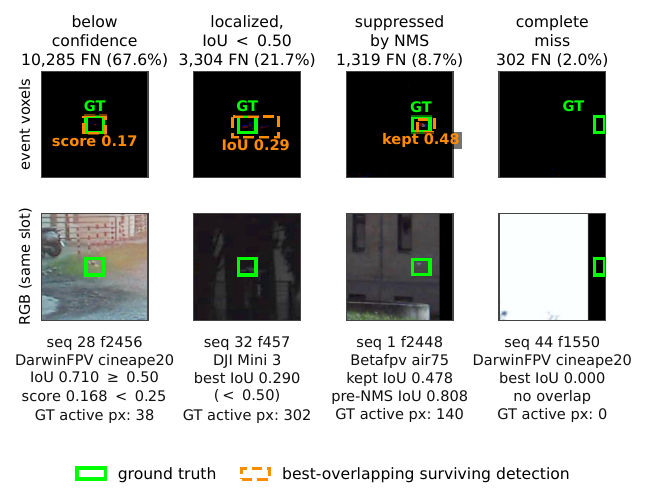}
\caption{The four principal false-negative causes of Figure~\ref{fig:error_breakdown}, one instance each, from the selected checkpoint on the development-validation split. Green: ground truth; orange: the best-overlapping candidate where one exists. Titles give each cause's share of all 15{,}215 false negatives. The below-confidence cell draws one of the 6{,}335 instances whose correct box survives non-maximum suppression, so the box shown is itself the sub-threshold detection; the remaining 3{,}950 of that group have no surviving box to draw. Every number printed on a cell is recomputed from the box drawn in that cell. The complete-miss cell contains no active pixel inside the ground-truth box, the cohort-wide property reported in Section~\ref{sec:qualitative}.}
\label{fig:qual_failures}
\end{figure*}

%% file: tables/tab_main_results.tex
\begin{table}[t]
\centering
\caption{Controlled comparison on FRED, in two blocks: development-validation estimates (all rows on the same frames, labels, evaluator, and post-processing) and the canonical test partition evaluated once, two seeds as two values. The YOLO11n rows are a system comparator, not an architecture-matched baseline; Section~\ref{sec:limits} states the asymmetries, Section~\ref{sec:main_results} the dispersion protocol. The resolution-matched control is a single seed on each arm and has no sealed row: the single permitted test evaluation was spent before it was scored. Per-sequence rows re-score the same frozen caches sequence by sequence; average precision is not additive over sequences, so the pooled margin is not their mean.}\label{tab:main_results}
{\footnotesize
\setlength{\tabcolsep}{3pt}
\begin{tabularx}{\linewidth}{@{}>{\raggedright\arraybackslash}Xcc>{\raggedleft\arraybackslash}X>{\raggedleft\arraybackslash}X@{}}
\toprule
Model & Processing & Modality & AP50 & AP50:95 \\
\midrule
\multicolumn{5}{@{}l}{\emph{Development-validation split (103{,}672 frames)}} \\
YOLO11n control (letterbox) & Dense & Event & 77.29 & 38.12 \\
YOLO11n control (resolution-matched) & Dense & Event & 84.68 & 42.44 \\
SparseVoxelDet & Sparse & Event & 87.01 & 43.14 \\
\quad Paired margin vs.\ letterbox, pooled & & & $+$9.72 & $+$5.02 \\
\quad \quad per-sequence min & & & $-$12.78 & $-$30.48 \\
\quad \quad per-sequence median & & & $+$3.83 & $+$3.06 \\
\quad \quad per-sequence max & & & $+$34.50 & $+$17.33 \\
\quad Paired margin vs.\ resolution-matched & & & $+$2.33 & $+$0.70 \\
\midrule
\multicolumn{5}{@{}l}{\emph{Canonical test partition, evaluated once (119{,}459 frames)}} \\
YOLO11n control (letterbox) & Dense & Event & 79.37 & 37.00 \\
SparseVoxelDet (seed 42) & Sparse & Event & 84.33 & 40.19 \\
\quad Paired margin, pooled & & & $+$4.96 & $+$3.19 \\
SparseVoxelDet (seed 123, partial replication) & Sparse & Event & 83.01 & 39.16 \\
\quad Paired margin, pooled & & & $+$3.64 & $+$2.17 \\
\bottomrule
\end{tabularx}
}
\end{table}

%% file: tables/tab_objective_fusion.tex
\begin{table}[t]
\centering
\caption{Supervision and fusion ablation on the development protocol. All rows share the SparseVoxelDet backbone, data, nominal 20-epoch schedule, evaluator, and selection rule; the topology-matched center-only row shares the quality-aligned row's exact 3-GPU distributed configuration, while the other center-only rows trained on one accumulating GPU (Section~\ref{sec:limits}). Head occupancy is the median realized share of stride-4 post-pool sites over the full development-validation split. The table holds one factor fixed at a time rather than crossing them: with supervision fixed at the center-only objective, the fusion swap cuts head occupancy from 78.88\% to 10.53\% and costs 0.60 AP50; with fusion fixed at the expansion-free architecture, quality-aligned supervision returns 1.49 AP50 at matched topology. The untrained factorial cell (quality-aligned supervision on transpose fusion) leaves the two changes not separated factorially; the final configuration ties for the sparsest head occupancy and is the most accurate of the four trained arms.}
\label{tab:objective_fusion}
\begin{tabular*}{\linewidth}{@{\extracolsep{\fill}}llccc@{}}
\toprule
Fusion & Supervision & AP50 & AP50:95 & Head occ.\ (\%) \\
\midrule
Transpose & Center-only & 86.31 & 40.54 & 78.88 \\
Inverse conv.\ & Center-only & 85.71 & 40.23 & 10.53 \\
Inverse conv.\ & Center-only (topology-matched) & 85.52 & 40.25 & 10.53 \\
Inverse conv.\ & Quality-aligned & \textbf{87.01} & \textbf{43.14} & 10.53 \\
\bottomrule
\end{tabular*}
\end{table}

%% file: figures/training_curves.tex
\begin{figure}[t]
\centering
\includegraphics[width=\linewidth]{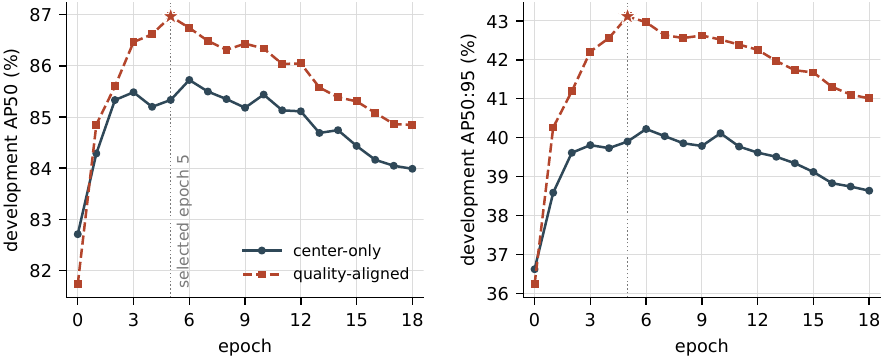}
\caption{Development-validation trajectories of the two supervision arms across the 19 epochs both arms scored, under the frozen evaluator. The quality-aligned arm leads at 18 of those 19 epochs on AP50 and on AP50:95; the single exception is epoch~0, before either arm has converged. The star marks the checkpoint the frozen selection rule returns: epoch~5 of the quality-aligned arm, released as \texttt{epoch\_005.pt}, and the dotted line carries it into the AP50:95 panel. The arms differ in optimizer topology as well as in supervision, so the comparison is matched by epoch and not by update count (Section~\ref{sec:limits}).}
\label{fig:training_curves}
\end{figure}

%% file: figures/error_forensics.tex
\begin{figure}[t]
\centering
\includegraphics[width=\linewidth]{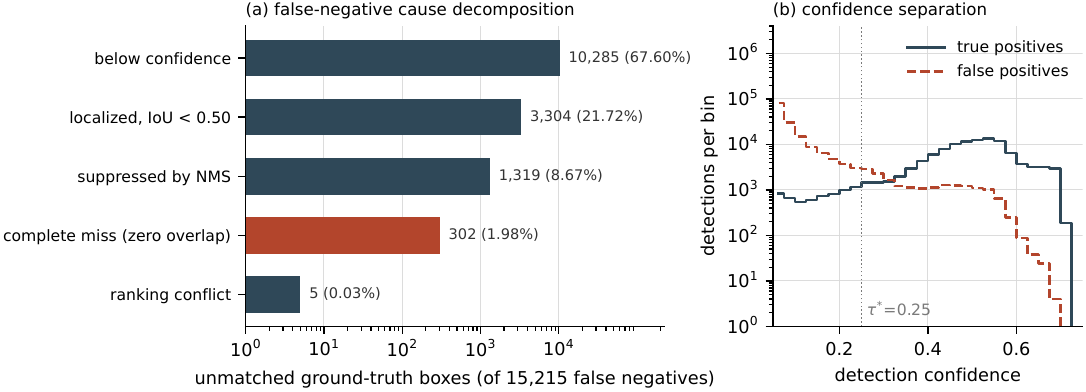}
\caption{Error forensics for the selected checkpoint on the 37-sequence development-validation split, under the frozen evaluator at score 0.05, class-agnostic NMS 0.5, and at most 100 detections per frame. (a)~Every unmatched ground-truth box at matching IoU 0.50 and $\tau^{*}=0.25$ is assigned exactly one cause under a strict priority order, so the five causes sum to the false-negative total exactly; the axis is logarithmic and the complete-miss bar is highlighted. (b)~Confidence histogram of true- and false-positive detections over the full scored set, with $\tau^{*}$ marked. The two populations cross once, low on the confidence axis: above the crossing the detector is right far more often than wrong, and the mass that costs recall lies below it.}
\label{fig:error_breakdown}
\end{figure}

%% file: figures/iou_threshold_analysis.tex
\begin{figure}[t]
\centering
\includegraphics[width=\linewidth]{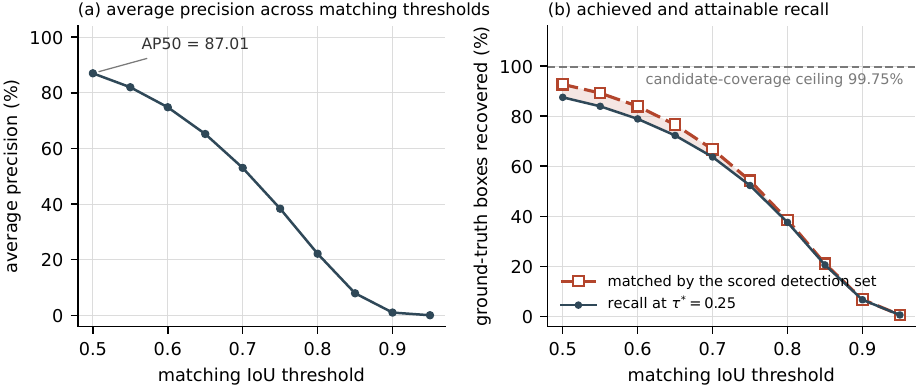}
\caption{Matching-threshold sensitivity of the selected checkpoint on the development-validation split under the frozen evaluator. (a)~Average precision across matching IoU thresholds from 0.50 to 0.95. (b)~Ground-truth boxes recovered at the reported operating point $\tau^{*}=0.25$, against the boxes the scored detection set already reaches before the confidence cut is applied; the shaded band is the share that the confidence threshold alone governs. Both panels are computed from the frozen offline prediction cache that scores the selected checkpoint at 87.01 AP50, the same path as Table~\ref{tab:main_results} and Table~\ref{tab:objective_fusion}. The dashed line is the candidate-coverage ceiling: 99.75\% of ground-truth boxes are overlapped by at least one candidate above the 0.001 pre-NMS score floor, so the zero-overlap cohort that no re-ranking and no threshold change can reach is 302 boxes of 121{,}982.}
\label{fig:iou_fn_breakdown}
\end{figure}

%% file: sections/07_limitations_ethics_repro.tex
\section{Limitations and Future Work}
\label{sec:limits}

\subsection{Limitations}

Dense-control asymmetries. The controlled comparison (SparseVoxelDet at 87.01/43.14 versus the letterbox dense control at 77.29/38.12) holds frames, source-complete labels, and evaluator fixed, but the two systems differ in three stated ways: the control is COCO-pretrained and selects its checkpoint from a longer schedule under its framework's own AP50:95-only rule (favoring the control, an asymmetry Section~\ref{sec:experiments} quantifies at 0.09 AP50); SparseVoxelDet carries 2.4$\times$ the parameters (6.22M versus 2.6M); and our anisotropic $640^2$ resizing gives it 1.78$\times$ the vertical input sampling of the isotropically letterboxed control (640 versus 360 signal rows). The comparison therefore supports a system-level statement rather than an architecture-causal one, and the control is a single seed. The resolution-matched control (Section~\ref{sec:main_results}) removes the input-sampling asymmetry: the dense control rises to 84.68 AP50 and the development margin narrows to $+$2.33, while the pretraining, schedule, and capacity asymmetries remain. At AP50:95 that margin is $+$0.70, smaller than the spread between our own two seeds (43.14 against 42.26, Section~\ref{sec:experiments}), so the resolution-matched accuracy case rests on the AP50 margin and the supervision ablation rather than on AP50:95. FRED's published 87.68 AP50 is an external reference, not an evaluator-matched comparison.

Supervision-ablation training configuration. The original center-only arm of Table~\ref{tab:objective_fusion} trained on one GPU with gradient accumulation over 4 steps (effective batch 8) while the quality-aligned arm trained with 3-GPU distributed data parallelism (global batch 6, no accumulation). The batch-matched confirmation has been run: a center-only arm retrained under the quality-aligned arm's exact distributed configuration reaches a best development AP50 of 85.52, within 0.19 of the single-GPU arm's 85.71, so the topology-matched supervision gain is 1.49 AP50. That rerun stopped at epoch 16 of its 20-epoch schedule when the trainer's preregistered non-finite-gradient gate tripped on a mixed-precision overflow that reproduced bit-identically on two resume replays; the arm's validation AP50 peaks at epoch 6 and declines monotonically through epoch 15; 85.52 is therefore the best of the 16 scored epochs, and the gain is conditional on that training budget.

Efficiency scope. Our efficiency evidence is within-family (expansion-free against transpose fusion on identical frames, hardware, and precision) plus the dense-execution control of Section~\ref{sec:efficiency}, whose dense arm is a mechanical dense realization of a network designed for coordinate-sparse execution and which runs on one card at batch size 1, the axis on which dense kernels recover most. We make no runtime, memory, compute, or energy claim against YOLO11n, whose kernels benefit from optimization that current sparse-convolution libraries do not match, and against published detectors we compare parameter count only. Energy was not measured, no measurement of ours was taken on an embedded platform, and no edge-deployment claim is made.

Residual support inflation. Expansion-free fusion bounds the pyramid's active set by the stored backbone supports, but the backbone's own strided convolutions still inflate support relative to the input (Section~\ref{sec:efficiency}). Fusion adds no sites, yet late-stage support does not return to input-level sparsity, and how backbone dilation scales with scene activity is open.

Single-class and object-count evaluation. The model detects one class (drone), and the 147/37 development protocol contains one- and two-object frames and no drone-free frames in either split, so the evaluation measures ranking and localization but not false-alarm behavior on empty sky. Extension to multi-class, zero-object, and higher object counts requires separate validation.

Development-protocol scope. Checkpoint selection, ablation, and every development number use the same 37-sequence development-validation split, so those values carry a selection-optimistic bias by construction. The single sealed test evaluation reports the resulting development-to-test change (the pooled margin narrows from 9.72 to 4.96 AP50; Section~\ref{sec:main_results}) without isolating selection bias from partition and distribution shift; the benchmark's \emph{challenging} partition remains unevaluated and unclaimed.

Ablation scope. The study ablates fusion mechanism and supervision (Table~\ref{tab:objective_fusion}); other components (temporal fusion strategy, channel depth, backbone depth, the per-sample squeeze-and-excitation block, the detection-head parameterization, and the composition of the regression loss) use a single configuration each. The supervision arm measures its change as a whole rather than isolating which of its three parts (inside-box assignment, soft quality targets, decoded-IoU quality head) carries the gain.

\subsection{Future Directions}

Learnable temporal fusion (sparse temporal attention or recurrent aggregation) could capture motion trajectories beyond what max-pooling preserves. Multi-scale heads at strides 4, 8, and 16 over expansion-free pyramids could extend the size range without reintroducing support growth. Layering scene-adaptive token selection~\citep{peng2024sast,yang2025smamba} atop the coordinate-sparse backbone would let attention prune what structural sparsity has preserved.
All results are reported on FRED. The coordinate-indexed interface accepts coordinate--feature pairs from any voxelization, but cross-class and cross-dataset benchmarking is required before generalizing beyond single-class drone detection.

\subsection{Ethical Considerations}

Drone detection has dual-use implications: it can protect critical infrastructure but could also enable surveillance. The analysis threshold we document (score 0.25, class-agnostic non-maximum suppression at IoU 0.5, at most 100 detections per frame, giving precision 0.854 at recall 0.875) is a development-split analysis point, as is the most permissive threshold on the analysis grid, 0.10, which reaches recall 0.915 at precision 0.651; neither is a deployment setting. Because neither development split contains drone-free frames, no threshold in this paper has been validated against empty-sky false alarms, and a deployment selects and validates its own threshold on recordings that contain both drones and drone-free intervals, measured as false positives per unit time. Detections should route to a human operator and remain decoupled from any autonomous engagement mechanism.

\subsection{Reproducibility}
\label{sec:repro}

All architecture code, training configurations, evaluation scripts, the error-forensics pipeline, the annotation-conversion stage that regenerates every reported label set from FRED's original release, and the stage-wise support profiler behind every occupancy and kernel-tap measurement are publicly available at \url{https://github.com/INQUIRELAB/SparseVoxelDet}; the FRED dataset is publicly available. Seed statistics are the two development-split values of Section~\ref{sec:experiments}, $87.01$ and $86.38$ AP50, reported as individual runs. Four SHA-256 digests pin the reported numbers to their inputs: the label manifest \texttt{6a973831}, the evaluator source \texttt{23c86324}, the selected checkpoint \texttt{b25c62a0}, and the development-validation prediction cache \texttt{2b5f4926} behind every forensics figure in Section~\ref{sec:error_forensics_results}. Full digests accompany the release.

%% file: sections/08_conclusion.tex
\section{Conclusion}
\label{sec:conclusion}

We presented SparseVoxelDet, a coordinate-sparse object detector for event cameras in which no stage, from backbone to detection head, constructs a dense spatial feature grid. Measuring the realized support at every stage exposed the central obstacle to that design, \emph{support inflation}: occupancy climbs from a median 0.0652\% at the input until the fused pyramid is locally near-dense, so naive coordinate sparsity densifies exactly where cost concentrates. Two design responses address it. Expansion-free inverse-convolution fusion restores coarse features only onto stored backbone supports, returning head occupancy from a median 78.88\% to 10.53\% by construction and cutting fusion-stage work by 91.3\%. Quality-aligned supervision then trains the same architecture to 87.01 AP50 and 43.14 AP50:95 on the source-complete development protocol, a 1.49 AP50 gain over center-only supervision with architecture, data, labels, evaluator, and optimizer topology held fixed, ahead of the same-label, same-evaluator dense controls at 77.29 AP50 letterboxed and 84.68 resolution-matched. Evaluated once on FRED's canonical test partition after every decision was frozen, the ordering holds: 84.33 and 83.01 AP50 on the two seeds against 79.37 for the dense control.

Preserving that sparsity is worth the effort because it is what the cost saving is made of. Running the selected checkpoint's own weights through numerically matched dense equivalents of its operators, sparse execution is cheaper on all 5{,}000 paired frames profiled, by a median 27.5$\times$ in work and 4.65$\times$ in latency, with the margin decaying from 7.05$\times$ to 1.41$\times$ as occupancy rises and never reaching parity. The saving is a property not of the representation alone but of how much support the network carries when it executes, which is precisely what expansion-free fusion controls.

The lesson generalizes: input occupancy is not a compute guarantee. A coordinate-sparse pipeline that fuses scales top-down either accepts support expansion or prevents it by construction, and preserving sparsity costs a small, measured 0.60 AP50 under center-only supervision that better supervision more than repays. Preserved by construction and supervised well, the sparse path keeps both the efficiency and the accuracy through detection.

%% file: sections/09_appendix_protocol.tex
\begin{appendices}

\section{Support-Efficiency Reporting Protocol}
\label{app:protocol}

A single input-occupancy figure can misrepresent realized cost by orders of magnitude (Section~\ref{sec:efficiency}). We propose that coordinate-sparse detectors report six fields over a stated evaluation split; the released profiler emits all of them in one pass.

\begin{enumerate}[leftmargin=*]
  \item \textbf{Per-stage realized support.} Active-site count and occupancy (median, mean, p95, max) at the input, every backbone and fusion stage, and the head.
  \item \textbf{Support-inflation factor and extraneous-support fraction.} $\iota_k$ and $\varepsilon_k$ (Section~\ref{sec:fpn}) for every fusion or upsampling operator.
  \item \textbf{Realized arithmetic.} Realized taps per active output and total realized multiply--accumulates, never conflated with site counts.
  \item \textbf{Batch-one latency distribution.} p50, p95, and p99 in the streaming regime, with warm-up rule, precision, and preprocessing scope stated.
  \item \textbf{Peak memory.} Peak activation memory at batch one, same conditions.
  \item \textbf{Execution stack.} GPU, driver, framework, and sparse-library versions; sparse wall-clock does not transfer across stacks.
\end{enumerate}

Fields 1--3 are properties of the architecture and data and transfer across hardware; fields 4--6 are stack-specific. Reporting them separately keeps architecture-causal claims distinct from implementation-dependent ones.

\section{Proof of Proposition~\ref{prop:support}}
\label{app:repro}
\label{app:proof}

\begin{proof}
The cached indice data fixes the output coordinate rows: weights, input feature values, and numerical cancellation affect only the stored feature vectors and neither add nor remove coordinate rows (a zero feature vector retains its row). The lateral submanifold projection computes output features on its input's coordinate array unchanged, so it preserves not only the coordinate set but the exact row order, the stronger property the inverse layer's precondition requires. That precondition is established at the backbone boundary rather than inherited from it. The squeeze--excitation and residual addition that follow each stride change can leave a stage's active set a proper subset of its stride-changing convolution's output rows, so each stage output an inverse convolution will consume (the stride-8 and stride-16 levels) is rebuilt onto the row order captured at that stage's stride-changing convolution before it leaves the backbone, with each dropped row restored carrying a zero feature vector. The first inverse therefore receives its input in the required order for the same reason the second does. Row-aligned sparse element-wise addition of tensors with structural supports $A, B \subseteq S_{k-1}$ yields structural support $A \cup B$; with $\mathrm{coords}(P_{k-1}) = S_{k-1}$ this union is $S_{k-1}$. The row-order rebuild restores the alignment precondition for the next inverse, so the recursion over $k{=}4,3$ gives the stride-4 statement.
\end{proof}

Two scope remarks: $D_k$ is the indice data cached by the stage's designated stride-changing convolution, so the coordinate map from level $k{-}1$ to $k$ is well defined; and the proposition formalizes the library's indice-cache contract, with the runtime certificate of Section~\ref{sec:fpn} establishing that the invariant holds in the shipped implementation.

The full executable specification behind Sections~\ref{sec:method} and~\ref{sec:experiments} (the fifteen SHA-256-pinned reimplementation sources covering voxel construction, dataset and augmentation, both fusion variants, both training objectives, trainer, evaluator, and label conversion; the label-provenance and timing conventions; the dense control's complete training recipe; and the topology-matched arm's build digests) is documented in the code release (Section~\ref{sec:repro}).

\end{appendices}